\documentclass[letterpaper]{article} 
\usepackage{aaai2026}  
\usepackage{times}  
\usepackage{helvet}  
\usepackage{courier}  
\usepackage[hyphens]{url}  
\usepackage{graphicx} 
\usepackage{natbib}  
\usepackage{caption} 
\usepackage{algorithm}
\usepackage{algorithmic}
\usepackage{amsfonts}
\usepackage{amsthm}
\usepackage{amsmath}

\usepackage{subcaption}
\usepackage{newfloat}
\usepackage{listings}
\usepackage{booktabs}
\usepackage{multirow}
\usepackage[table]{xcolor}
\usepackage{xspace}
\usepackage{newfloat}
\usepackage{listings}
\DeclareCaptionStyle{ruled}{labelfont=normalfont,labelsep=colon,strut=off} 
\floatstyle{ruled}
\newfloat{listing}{tb}{lst}{}
\floatname{listing}{Listing}
\newcommand{\algoname}[1]{{\fontfamily{cmtt}\selectfont#1}}
\newtheorem{proposition}{Proposition}
\newtheorem{lemma}{Lemma}
\newcommand{\Catalonia}{Catalonia\xspace}
\newcommand{\RisCanvi}{RisCanvi\xspace}
\newcommand{\citeriscanvi}{\cite{riscanvi}}
\newcommand{\citeriscanvicb}{\cite{Dribia2024RisCanviAudit}}
\newcommand{\citecourtuse}{\cite{JimenezArandia2026RisCanvi}}
\newcommand{\spara}[1]{\smallskip\noindent\textbf{#1}}

\title{Model Multiplicity and Predictive Arbitrariness in Recidivism Risk Assessment}
\author{
    Ashwin Singh\textsuperscript{\rm 1},
    Carlos Castillo\textsuperscript{\rm 2}
}
\affiliations {
    \textsuperscript{\rm 1}TU Wien, Vienna, Austria\\
    \textsuperscript{\rm 2}ICREA and Universitat Pompeu Fabra, Barcelona, Spain\\
    ashwin.singh@tuwien.ac.at, chato@icrea.cat 
}

\begin{document}

\maketitle

\begin{abstract}
Prediction tasks over individual futures, which are inherently noisy, often admit multiple similarly accurate models. When these models produce different predictions for the same individual, they raise concerns of arbitrariness in decision-making. How severe can this arbitrariness be, in theory and in practice? How can it be resolved to support high-stakes risk assessment? 
We address these questions through a study of a machine learning-based decision support system for recidivism risk assessment that has been in use for over 15 years. By translating complex legal rules into an algorithm for labeling post-release outcomes (\emph{recidivist} or \emph{non-recidivist}), we first construct a dataset of thousands of inmate releases. Using this dataset, we learn interpretable models that improve predictive performance, reduce error-rate disparities between groups, and ensure that rehabilitative progress lowers risk scores.
Next, we study predictive multiplicity, by first deriving a tight lower bound on the expected predictive agreement of any finite set of models over a dataset, and then by evaluating the extent to which structural diversity (e.g., different model coefficients) within this set translates to predictive multiplicity (i.e., different predictions for the same individual).
Our experiments indicate that the existence of many similarly accurate models with comparable error-rate disparities does not necessarily translate into severe predictive multiplicity. Empirically, similarly performant models can exhibit substantially higher predictive agreement than worst-case theoretical guarantees suggest. We find that a simple policy that assigns each inmate the lowest risk among these models is effective for addressing predictive arbitrariness.
\end{abstract}

\section{Introduction}
The use of machine learning to support risk assessments of criminal recidivism is one of the most well-studied applications in algorithmic fairness research. 
Increased interest on this topic over the past decade can be traced back to a study of the \algoname{COMPAS} algorithm that has been a subject of considerable debate and research~\cite{angwin2016machine,rudin2020age,Jackson2020Setting,bao2021compaslicated}.
Since the publication of these studies, recidivism prediction has arguably taken the role of a ``model organism'' for research in this area, similarly to well-studied species that are selected for intensive examination in biology research, or to platforms that have taken that role in social computing research~\cite{tufekci2014big}.

Considering the high-stakes nature of recidivism risk assessment, the European Union's AI Act~(\citeyear{ai_act_2024}) categorizes it as a high-risk application of AI.  
To comply with the AI Act, high-risk AI systems must, among other requirements, demonstrate an appropriate level of accuracy, robustness, and resilience to inconsistencies.

In practice, these criteria are difficult to satisfy.
Obtaining \textbf{reliable training data} requires determining if people released in the past recidivated, which is legally complex.
For instance, due to long judicial processes and sentencing delays, individuals sometimes serve a prison sentence only to be re-incarcerated later for a crime committed before their first term. 
Hence, returning to prison does not always indicate recidivism. This means labeling needs to be done either manually (delaying the extraction of up-to-date training data) or by applying a complex set of rules~\cite{karimihaghighi2022risk}.
While some systems circumvent this issue by using \emph{re-arrests} as proxy labels---such as \algoname{COMPAS}~\cite{Jackson2020Setting} or  \algoname{OASYS}~\cite{OASYS}---this is problematic for two reasons.
First, the distinction between an \emph{arrest} and \emph{conviction} is best captured in the scrutiny of judicial due process. Second, arrest rates are often unequal across demographic groups. Using \emph{re-arrest} labels can therefore reinforce bias through feedback loops, a risk the AI Act explicitly requires high-risk systems to mitigate.

Even with ground truth labels, \textbf{consistency} in recidivism risk assessment is difficult to achieve.
There are often many equally accurate models that assign different predictions to the same individuals.
This has been dubbed the \emph{Rashomon Effect}:\footnote{The name ``Rashomon'' refers to the 1950 movie by director Akira Kurosawa, in which multiple people offer credible but mutually incompatible accounts of the same crime.} when different models for the same hypothesis class perform similarly well on a prediction task~\cite{breiman}.
The set of all such models is  called a \emph{Rashomon Set}, and presents an opportunity to address algorithmic fairness concerns.
One can, for instance, search within a \emph{Rashomon Set} for models that satisfy desirable properties such as statistical non-discrimination criteria and monotonicity constraints~\cite{JMLR:v20:18-760, pmlr-v139-coston21a, pmlr-v235-rudin24a}. 
Even then, predictive disagreement among such models raises concerns of \textbf{arbitrariness} in the decision process.
While ensembling and randomization have been proposed to resolve predictive multiplicity~\cite{multiplicityFAccT}, both approaches have important limitations. Ensembling reduces interpretability, which is especially important in a high-stakes context, whereas randomization is often perceived unfavorably by stakeholders~\cite{CHImultiplicity}.
More broadly, much of the literature on model multiplicity does not engage with its implications in real-world settings~\cite{AIESmultiplicity}. \\

\spara{Contributions. }Our work seeks to address these concerns simultaneously.
We study \algoname{\RisCanvi}, a decision support system for recidivism risk assessment that has been operational across prisons in \Catalonia since 2009.
By translating a complex set of legal rules into an algorithm for determining post-release outcomes (\emph{recidivist} or \emph{non-recidivist}), we construct a dataset of over 17.5K inmate releases between 2010 and 2019. Using this dataset, which is $17\times$ larger than the one \algoname{\RisCanvi} is currently trained on, we build on the MILP (mixed-integer linear programming) \algoname{SLIM} formulation of \citet{Ustun2016} and the Rashomon Set exploration framework of \citet{enumfreexplor}. This allows us to learn models that are not only significantly more accurate than the one in use, but also overcome important limitations. In particular, they exhibit lower error-rate disparities between subgroups, and ensure rehabilitative progress translates into lower risk scores through monotonicity constraints. Finally, we present a simple yet effective policy for addressing predictive arbitrariness in the set of all such models. Our \textbf{main contributions} are as follows:

\begin{enumerate}
\item We generalize the search for less discriminatory models in the Rashomon Set to \(m>2\) subgroups by adding only \(O(m)\) constraints to the MILP of \citet{enumfreexplor}.
\item We extend the notion of self-consistency introduced by \citet{arbitraryAAAI} to Rashomon Sets and derive a tight lower bound on expected self-consistency for any finite set of binary classifiers over a dataset. 
\item We present the first (to the best of our knowledge) real-world case study of predictive multiplicity in recidivism risk assessment.
\item We propose and study a \emph{lowest-risk policy} for addressing predictive arbitrariness, with a detailed discussion of its institutional and legal implications in recidivism risk assessment.
\end{enumerate}

The rest of this paper is organized as follows.
The next section provides a background on \algoname{\RisCanvi}, on non-discrimination criteria for supervised learning, and on predictive multiplicity in Rashomon Sets (\S~\ref{sec:background}).
Then, we describe our optimization framework, experimental setup, and evaluation strategy (\S~\ref{sec:methods}).
Next, we present our results (\S~\ref{sec:results}) and discussion (\S~\ref{sec:discussion}).
The last section presents our conclusions and the limitations of this study (\S~\ref{sec:conclusions}).

\section{Background}
\label{sec:background}

In this section, we first provide background about the design goals and features of \algoname{\RisCanvi}, and situate its use as a decision support tool for risk assessment (\S~\ref{subsec:riscanvi}).
We then introduce technical preliminaries on supervised learning and statistical non-discrimination criteria (\S~\ref{subsec:nondiscrimsuper}). 
Finally, we give an overview of related work on Rashomon Sets and the predictive multiplicity that arises from them (\S~\ref{subsec:rashomon}).

\subsection{\RisCanvi}
\label{subsec:riscanvi}

\algoname{\RisCanvi} is a decision support tool for risk assessment that has supported management of sentencing conditions and the provision of prison alternatives to inmates in \Catalonia for more than 15 years~\citeriscanvi. It uses logistic regression with two types of features to predict recidivism risk:
(i) $20$ \textbf{static} features that describe immutable characteristics of the inmate (e.g., age at the onset of criminal activity), and
(ii) $23$ \textbf{dynamic} features that can vary over time, and in principle, should lead to risk reduction due to rehabilitative progress (e.g., level of education).
An overview of features used by \algoname{\RisCanvi} can be found in Appendix~\S~\ref{supp:riscanvi-features}. Thresholds over the predicted risk are then used to classify inmates into one of three risk levels (low, medium and high). Assigned risk levels support evaluations by the prison staff, including social workers, psychologists, and  lawyers, which are considered in judicial decisions related to requests for parole, in changes to sentence conditions (``degree of imprisonment''), and in the design of rehabilitation programs.

\algoname{\RisCanvi} is trained on labels that indicate whether an inmate recidivated within five years of being released. Specifically, it considers \emph{penal recidivism}, i.e., when a person is sentenced to prison and actually enters prison to serve that prison sentence. However, reliably constructing these labels at scale is challenging. This is because penal recidivism is determined by a complex set of legal rules and applying these rules is a time-consuming process susceptible to human error. Moreover, due to limited institutional resources, this labeling process is not carried out regularly. As a result, \algoname{\RisCanvi} often lacks up-to-date training data.

In principle, all prison inmates should undergo a \algoname{\RisCanvi} evaluation once every six months. In reality, longitudinal evaluation records for most inmates have longer gaps.
Only close to the end of their prison sentence do we find that complete evaluations---having all features---are reliably found. This is partly because the law in \Catalonia mandates risk assessment before an inmate is released from prison.
In addition, other concerns about \algoname{\RisCanvi} have been documented by external audits. Most notably, the system has a high false negative rate, and exhibits significant false positive rate disparities between nationals and foreigners~\citeriscanvicb. 

\subsection{Non-Discrimination in Supervised Learning}
\label{subsec:nondiscrimsuper}

Let $\mathcal{X}\subseteq\mathbb{R}^{p}$ and $\mathcal{Y}=\{-1,1\}$ denote the feature space and label space respectively. In our case, $y=1$ implies \emph{recidivist} whereas $y=-1$ implies \emph{non-recidivist}. A dataset consists of labeled examples $\mathcal{S}=\{(\mathbf{x}_i,y_i)\}_{i=1}^n$ where $\mathbf{x}_i\in\mathcal{X}$ is the feature vector for instance $i$ and $y_i \in \mathcal{Y}$ is the corresponding label. In supervised learning, the goal is to learn a classifier $h:\mathcal{X}\to\mathcal{Y}$ from a specified hypothesis class $\mathcal{H}$ that minimizes a loss function
$\ell:\mathcal{Y}\times\mathcal{Y}\to\mathbb{R}_{\ge 0}$ over $\mathcal{S}$. In binary classification, a standard choice is the $0$-$1$ loss, $\ell(h(\mathbf{x}),y)=\mathbf{1}[h(\mathbf{x})\neq y].$ The empirical risk of a classifier $h$ on a dataset $\mathcal{S}$ is then
$L_{\mathcal{S}}(h)=\frac{1}{n}\sum_{i=1}^n \mathbf{1}[h(\mathbf{x}_i)\neq y_i].$ 

In our work, we consider the hypothesis class \(\mathcal{H}\) of sparse linear integer models as they are widely used for recidivism risk assessments~\cite{HOFFMAN1994477, strongr}. It consists of models that predict \(h(\mathbf{x})=1\) if \(\lambda_h^\top \mathbf{x}\geq \gamma\), and $-1$ otherwise. Here, \(\lambda_h\in\mathbb{Z}^p\) denotes the integer coefficient vector and \(\gamma\in\mathbb{R}\) denotes the decision threshold. 

Let $\mathcal{G} = \{G_1,    \dots,G_m\}$ denote the partition of $\mathcal{S}$ into demographic groups induced by the intersections of sensitive attributes, and let $g:\mathcal{X}\to\mathcal{G}$ be the group membership function.
A number of \textbf{statistical non-discrimination criteria} for supervised learning have been proposed in the literature; they typically equalize a group-dependent statistic (see, e.g., \citeauthor{fairmlbook} \citeyear{fairmlbook}).
\emph{Statistical Parity}~\cite{statisticalparity} equalizes prediction rates across all pairs of groups $G_i,G_j \in \mathcal{G}$:
\[\mathbb{P}(h(\mathbf{x})=1 \mid g(\mathbf{x})=G_i) = \mathbb{P}(h(\mathbf{x})=1 \mid g(\mathbf{x}) = G_j)\] 
This criterion can be insufficient when groups have unequal base rates, which is often the case with criminal recidivism. \emph{Equalized Odds}~\cite{equalizedodds} addresses this limitation by conditioning a prediction $h(\mathbf{x})$ on its true outcome $y$. In particular, it requires the following to hold for all $y \in \mathcal{Y}$ and $G_i,G_j \in \mathcal{G}$:
\[\mathbb{P}(h(\mathbf{x})=1 \mid g(\mathbf{x})=G_i, y)=\mathbb{P}(h(\mathbf{x})=1 \mid g(\mathbf{x}) = G_j, y).\]

\subsection{Rashomon Sets and Predictive Multiplicity}\label{subsec:rashomon}
The \emph{Rashomon Effect}, i.e., the existence of multiple models with equivalent prediction accuracy, is prevalent in noisy prediction tasks such as criminal recidivism and commonly formalized through a \emph{Rashomon Set}~\cite{simplepathnoise}.
Let $h_\mathcal{S}=\arg\min_{h \in \mathcal{H}} L_{\mathcal{S}}(h)$ be the empirical risk minimizer. Then for $\epsilon \geq 0$, the $\epsilon$-Rashomon Set consists of all models in $\mathcal{H}$ whose empirical risk lies within $\epsilon$ of $h_\mathcal{S}$~\cite{JMLR:v20:18-760}. Formally:
\[
\mathcal{R}(\mathcal{H},\epsilon,\mathcal{S}) = \{h \in \mathcal{H}\;:\;L_{\mathcal{S}}(h)\le L_{\mathcal{S}}(h_\mathcal{S}) + \epsilon\}.
\]
Enumerating the Rashomon Set is generally intractable. Accordingly, there is a growing body of work on methods for exploring it. Broadly, these methods either enumerate the set, sample from it, or search it for models with desired properties. Enumeration-based methods recover the complete set, which is feasible for discrete hypothesis classes such as sparse decision trees and rule sets~\cite{TreeFarms2022, rulesetexplor2024}. 
Sampling-based methods trade completeness for scalability, generating near-optimal model sets from both discrete and continuous hypothesis classes. Examples include sparse linear integer models obtained via beam
search~\cite{fasterrisk2022} and neural networks obtained via dropout-based techniques~\cite{dropouthsu2024}. 
Finally, search-based methods optimize over the Rashomon set to find models with properties of interest. Two notable goals of this optimization are (i) finding less discriminatory models (LDMs) within the Rashomon Set~\cite{ldagillis, enumfreexplor}, and (ii) resolving arbitrariness in decision-making due to predictive multiplicity~\cite{pmlr-v119-marx20a}. 
Regarding the first goal, finding LDMs is important because a model sampled uniformly at random from the Rashomon Set may be significantly more discriminatory than the least discriminatory one~\cite{multEAAMO}. Regarding the second goal, arbitrariness can arise not only from multiplicity within the Rashomon Set, but also from variation in the training data or learning procedure~\cite{loneout}. \citet{arbitraryAAAI} introduce \emph{self-consistency} to measure how stable a prediction is across models trained with the same learning procedure on equally sized training samples. Formally, for an instance $\mathbf{x}$:
\[
SC(\mathcal{A},\mathbb{S},\mathbf{x}) = \mathbb{E}_{h_{S_i},h_{S_j}\sim\mu}\left[\mathbf{1}[h_{S_i}(\mathbf{x})=h_{S_j}(\mathbf{x})]\right]
\]
\[
= \mathbb{P}_{h_{S_i},h_{S_j}\sim\mu}(h_{S_i}(\mathbf{x})=h_{S_j}(\mathbf{x})).
\]
where \(\mathbb{S}\) denotes the set of all equally sized subsets of \(\mathcal{S}\), and \(\mu\) is the distribution over models generated by applying \(\mathcal{A}\) to samples \(S \in \mathbb{S}\). In their framing, low self-consistency corresponds to a greater degree of arbitrariness. 

Our work builds on search-based methods for exploring Rashomon Sets, and on self-consistency for studying the predictive multiplicity within them. Methodologically, our optimization framework (\S~\ref{subsec:methods-optimization}) extends the MILP-based search of~\citet{enumfreexplor} to find LDMs across $m > 2$ groups under monotonicity constraints. To study the arbitrariness among the resulting models, we extend the notion of self-consistency to Rashomon Sets and derive a tight lower bound on its expectation over a dataset $\mathcal{S}$ (\S~\ref{subsec:predmul}).

\section{Methods}
\label{sec:methods}

In this section, we first introduce the MILP-based optimization framework with our extensions (\S~\ref{subsec:methods-optimization}), then describe our dataset, experimental setup, baselines, and evaluation metrics (\S~\ref{subsec:methods-experimental}), followed by our extension of self-consistency~\cite{arbitraryAAAI} to study predictive multiplicity in finite model sets (\S~\ref{subsec:predmul}).

\subsection{Optimization Framework}
\label{subsec:methods-optimization}
We use a two-step optimization framework. First, we learn a sparse linear integer model that enforces monotonicity constraints on how static and dynamic features shape risk scores. Then, we use this model to initialize a search for less discriminatory models within its Rashomon Set.
\subsubsection{SLIM with Monotonicity Constraints.} 
The \algoname{SLIM} (\emph{Supersparse Linear Integer Model}) MILP, given by ~\citet{Ustun2016}, is designed to learn linear scoring systems with integer coefficients. We use it to find an empirical risk-minimizer $h_\mathcal{S}$ for recidivism prediction, using balanced 0--1 loss as the objective. A key advantage of this approach is that modern MILP solvers either certify optimality or report an optimality gap, i.e., the difference between the incumbent feasible solution and the best bound on the objective. 

Moreover, we enforce monotonicity constraints in \algoname{SLIM} by restricting feature coefficients to non-negative integers. These constraints are important from a practical standpoint because they guarantee that changes in recidivism risk align with expectations of the prison staff. Specifically, predicted risk should increase in the presence of more adverse indicators (static features) and decrease based on responsiveness to rehabilitation (dynamic features). 

\subsubsection{Finding Less Discriminatory Models.}
The framework proposed by~\citet{enumfreexplor} uses an empirical risk minimizer \(h_\mathcal{S} \in \mathcal{H}\) to initialize the MILP that explores the corresponding $\epsilon$-Rashomon Set. We use the solution returned by our extended \algoname{SLIM} formulation and build on their MILP formulation for scoring systems. In particular, we modify it to minimize the maximum pairwise error-rate disparity across \(m>2\) groups. Let
\[
\Delta_{\mathrm{FPR}}(h)
=
\max_{G_i,G_j\in\mathcal{G}}
\left|
\mathrm{FPR}_{G_i}(h)-\mathrm{FPR}_{G_j}(h)
\right|
\]
and
\[
\Delta_{\mathrm{FNR}}(h)
=
\max_{G_i,G_j\in\mathcal{G}}
\left|
\mathrm{FNR}_{G_i}(h)-\mathrm{FNR}_{G_j}(h)
\right|.
\]
Our objective minimizes the \emph{Equalized Odds} violation:
\[
\Delta_{\mathrm{EO}}(h)
=
\max\left\{
\Delta_{\mathrm{FPR}}(h),
\Delta_{\mathrm{FNR}}(h)
\right\}.
\]

Observe that the largest error-rate disparity between any two groups does not require comparing every pair explicitly. For false positive rates, the worst pair is simply the group with the highest FPR and the group with the lowest FPR; their difference equals the maximum pairwise FPR disparity. The same holds for false negative rates. We therefore introduce auxiliary variables that track the minimum and maximum FPR and FNR across groups, and minimize the larger of the two resulting gaps. This reduces the number of fairness constraints from \(O(m^2)\), for all pairwise group comparisons, to \(O(m)\), for \(m\) groups. Additionally, we reuse the monotonicity constraints from the \algoname{SLIM} formulation. We refer to the resulting formulation as \algoname{FairSLIM}, which is included in Appendix~\S~\ref{supp:fairslim-milp}.

\subsection{Experimental Setup}
\label{subsec:methods-experimental}
We now describe in detail how we construct our dataset, the chronological train-test split, and the MILP configuration used to instantiate the optimization framework. 

\subsubsection{Iterative Algorithm Design.} To identify recidivism reliably and at scale, we develop a labeling algorithm through an iterative process in close collaboration with researchers specializing in sentence enforcement.\footnote{These researchers work at the Department of Justice (DoJ) of Catalonia, which operates \algoname{RisCanvi}. They do not share any affiliation with the authors, and were not involved in any modeling decisions in this work. They are also distinct from the practitioners who produced \algoname{RisCanvi}'s original training labels. However, both groups apply the same protocol prescribed by the DoJ for labeling penal recidivism.} We begin by translating the legal rules that define penal recidivism into an initial rule-based algorithm. In each subsequent iteration, we share the labels assigned by the algorithm with the researchers. In turn, they annotate a random sample of this data and provide justifications for the incorrectly labeled instances. Then, we incorporate these justifications to refine the rule set of the algorithm. This process is repeated until labels assigned by the algorithm fully match those of the researchers. 
We validate the correctness of our proposed algorithm by comparing its labels against ground-truth labels for all releases between 2010 and 2015. Because our algorithm exactly reproduces the labels prescribed by the same protocol used to label RisCanvi's training data, the two label sets are commensurable. Upon validation, we apply our algorithm to the remaining releases to construct our dataset. Appendix~\S~\ref{supp:labelingalgo} provides a condensed version of the pseudo code for the labeling algorithm. 

\subsubsection{Dataset Description.}
The resulting dataset contains 17.5K releases from prisons in \Catalonia between 2010 and 2019. Table~\ref{tab:dataset_description} reports the number of releases and recidivism rates across groups defined by age, sex, and nationality. Age groups are split at $30$, a common cutoff in criminology that reflects differences in types of crimes and social contexts between younger and older adults~\cite{agegroups}. As described in \S~\ref{subsec:riscanvi}, each release is associated with 43 features from the inmate's last evaluation, typically conducted six to nine months before release.

\begin{table}[h]
\centering\small
\caption{Distribution of recidivists and non-recidivists across groups based on sensitive attributes.}
\label{tab:dataset_description}
\setlength{\tabcolsep}{6pt}
\begin{tabular}{lrr}
\toprule
\textbf{Group} & \textbf{\# of Releases} & \textbf{Recidivism Rate} \\
\midrule
Age $\geq 30$ & 14,049 & 0.36 \\
Age $< 30$ & 3,493 & 0.50 \\
\midrule
Female & 655 & 0.47 \\
Male & 16,887 & 0.38 \\
\midrule
National & 11,785 & 0.37 \\
Foreigner & 5,757 & 0.42 \\
\midrule
Overall & 17,542 & 0.39 \\
\bottomrule
\end{tabular}
\end{table}

\subsubsection{Train-Test Split.} We use a chronological train-test split. The train set includes releases from fully observed years 2010--2014 as well as recidivists from partially observed years (2015--2018) that precede the test set year, 2019. Non-recidivists for release years 2015--2018 cannot be ascertained until after the test year, and are therefore excluded from training. The test set consists of all inmates released in 2019, whose five-year observation period concluded in 2024. Table~\ref{tab:split_stats} summarizes this split. We did preliminary experiments using training data solely from the fully observed years, and also probabilistically from partially observed years, but observed a consistently lower accuracy than in this set-up (results omitted for brevity).
\begin{table}[t]
\centering\small
\caption{Overview of train-test split. Releases from train years 2015--2018 include only observed recidivists, since non-recidivists cannot be confirmed before the test year. For instance, a person released after 2015 may recidivate in 2020, which is after the test year 2019.}
\begin{tabular}{llrr}
\toprule
\textbf{Split} & \textbf{Years} & \textbf{\# of Releases} & \textbf{Recidivism Rate} \\
\midrule
\multirow{2}{*}{Train}
& 2010--2014 & 8,177 & 0.44 \\
& 2015--2018 & 1,958 & 1.00 \\
\midrule
Test & 2019 & 1,881 & 0.29 \\
\bottomrule
\end{tabular}
\label{tab:split_stats}
\end{table}
\subsubsection{MILP Configuration.}
Both \algoname{SLIM} and \algoname{FairSLIM} are implemented in Python using the Gurobi solver~\cite{gurobi}. In \algoname{SLIM}, we use a finite coefficient grid that includes small incremental weights as well as larger values for highly predictive features. Since all features are processed so that larger values indicate higher risk, their coefficients $\lambda$ are chosen from a set of non-negative values $\mathcal{L} =\{0,1,2,3,4,5,6,7,8,9,25,50,100\}$ to enforce monotonicity. Intercept $\lambda_0$ is allowed to take both positive and negative values from the signed version of $\mathcal{L}$. Since all features take values in \(\{0,0.5,1\}\) and coefficients are integers, every score lies on a grid with spacing \(0.5\). Thus any correctly classified point has margin at least \(0.5\) and any choice of \(\gamma \in (0,0.5]\) is valid. We set \(\gamma=0.1\). For robustness, we run \algoname{SLIM} for ten random seeds (four hours per seed). Although the solver does not converge within the allotted time, the majority of improvement in its objective occurs within the first 30 minutes for all seeds, with only marginal gains thereafter (details in Appendix~\S~\ref{supp:convergeruntime}). 

For every seed $s \in \{1,\dots,10\}$, we use the incumbent feasible solution $h_s$ returned by \algoname{SLIM} to instantiate the \algoname{FairSLIM} MILP, and explore its Rashomon Set $\mathcal{R}(\mathcal{H},\epsilon,\mathcal{S})$ over $\epsilon  \in \{.01,.02,.03,.04,.05\}$. Each run is allowed 2 hours per $\epsilon$ per seed. As with \algoname{SLIM}, \algoname{FairSLIM} does not reach provable optimality within the allotted time. However, its incumbent objective stabilizes well before the time limit is reached, and exhibits much smaller optimality gaps compared to \algoname{SLIM}. During this search, we retain all models whose fairness objective $\Delta_{\mathrm{EO}}$ lies within 5\% of the best solution obtained by \algoname{FairSLIM}. We use $\mathcal{P}_s$ to denote the set (or pool) of all such models retained for seed $s$; the number of models in $\mathcal{P}_s$ never exceeds 10 (on average 3.9 ± 2.6).

\begin{table*}[t!]
\centering\small
\caption{Predictive performance and algorithmic fairness metrics. Results for \algoname{CatBoost} and \algoname{SLIM} are averaged over ten seeds and reported as mean $\pm$ standard deviation. We highlight our proposed Lowest-Risk Policy (\algoname{FairSLIM-LRP}) which is applied to the aggregated model pool $\mathcal{P}$ with $\epsilon\leq.01$. The performance of \algoname{\RisCanvi} on the train set is computed using only the fully observed years (2010--2014); otherwise its accuracy drops substantially (to 49.7\%) due to errors in predicting the positive class.}
\label{tab:performance_train_test}
\setlength{\tabcolsep}{5pt}
\begin{tabular}{llccccccc}
\toprule
\textbf{Split} & \textbf{Model} & \textbf{F1 Score} & \textbf{Accuracy} & \textbf{Balanced Accuracy} & \textbf{FPR} & \textbf{FNR} & $\mathbf{\Delta_{\mathrm{FPR}}}(h)$ & $\mathbf{\Delta_{\mathrm{FNR}}}(h)$ \\
\midrule
\multirow{2}{*}{Train}
& \algoname{RisCanvi}
& $25.6$
& $60.8$
& $56.2$
& $2.9$
& $84.7$
& $4.9$
& $14.5$ \\

& \algoname{CatBoost} 
& $80.9 \pm 0.1$ 
& $78.5 \pm 0.1$ 
& $78.1 \pm 0.1$ 
& $26.2 \pm 0.1$ 
& $17.7 \pm 0.1$ 
& $17.8 \pm 0.7$ 
& $11.5 \pm 0.3$ \\

& \algoname{SLIM}
& $78.1 \pm 0.3$ 
& $76.0 \pm 0.4$ 
& $75.8 \pm 0.6$ 
& $26.1 \pm 1.9$ 
& $22.4 \pm 1.1$ 
& $30.1 \pm 2.0$ 
& $15.7 \pm 1.9$ \\

& \cellcolor{green!10}\algoname{FairSLIM-LRP}
& \cellcolor{green!10}$71.1$
& \cellcolor{green!10}$72.3$
& \cellcolor{green!10}$73.6$
& \cellcolor{green!10}$14.5$
& \cellcolor{green!10}$38.4$
& \cellcolor{green!10}$14.2$
& \cellcolor{green!10}$8.9$ \\
\midrule
\multirow{4}{*}{Test}
& \algoname{RisCanvi}
& $24.5$
& $71.5$
& $55.3$
& $5.2$
& $84.3$
& $7.5$
& $15.0$ \\

& \algoname{CatBoost}
& $58.7 \pm 0.2$ 
& $63.6 \pm 0.2$ 
& $70.7 \pm 0.2$ 
& $46.6 \pm 0.2$ 
& $12.1 \pm 0.3$ 
& $11.4 \pm 0.4$ 
& $5.4 \pm 0.5$ \\

& \algoname{SLIM}
& $58.0 \pm 0.6$ 
& $64.6 \pm 0.8$ 
& $70.0 \pm 0.6$ 
& $43.0 \pm 1.4$ 
& $17.0 \pm 0.9$ 
& $2.5 \pm 2.2$ 
& $5.9 \pm 3.2$ \\

& \cellcolor{green!10}\algoname{FairSLIM-LRP}
& \cellcolor{green!10}$58.5$
& \cellcolor{green!10}$71.2$
& \cellcolor{green!10}$70.6$
& \cellcolor{green!10}$27.8$
& \cellcolor{green!10}$31.1$
& \cellcolor{green!10}$0.5$
& \cellcolor{green!10}$5.9$ \\
\bottomrule
\end{tabular}
\end{table*}

\subsubsection{Baselines.} We compare against two baselines: (i) the currently operational \algoname{\RisCanvi} model which is based on logistic regression, and (ii) \algoname{CatBoost}~\cite{catboost}, a gradient-boosted tree method recommended by a recent third-party audit as a replacement for \algoname{\RisCanvi}~\citeriscanvicb. As with \algoname{SLIM}, monotonicity constraints for features are also enforced in \algoname{CatBoost}.

\subsubsection{Performance Metrics.} We report F1 Score, Accuracy, and Balanced Accuracy for all models to assess their predictive performance.
To evaluate discrimination, we report the maximum violation of error-rate parity across any two groups ($\Delta_{\mathrm{FPR}}(h)$ and $\Delta_{\mathrm{FNR}}(h)$). All metrics are averaged over $10$ random seeds, and reported with standard deviation.

\subsection{Predictive Arbitrariness and Self-Consistency}\label{subsec:predmul}
Our search for less discriminatory models using \algoname{FairSLIM} yields a model pool $\mathcal{P}_s \subseteq \mathcal{R}(\mathcal{H},\epsilon,\mathcal{S})$ for every seed $s$. To study predictive multiplicity, we consider the aggregated pool over the ten seeded runs \(\mathcal{P}=\bigcup_{s=1}^{10}\mathcal{P}_s\). We do so for two reasons. First, predictive multiplicity is a property of many competing models, and our seed-level pools $\mathcal{P}_s$ are too small to support meaningful evaluations of arbitrariness. Second, each seed-level pool $\mathcal{P}_s$ is obtained by initializing the solver with a different \algoname{SLIM} incumbent $h_s$. Aggregating across seeds therefore gives a larger and possibly more diverse set of less discriminatory models.

All models in $\mathcal{P}$ have comparable predictive performance and error-rate disparity on the train set. Therefore, any of these models could be a plausible candidate for deployment. 
To evaluate the resulting arbitrariness in risk assessment, we extend the notion of self-consistency~\cite{arbitraryAAAI} to Rashomon Sets, and more generally, to any finite set of models. Formally, we define the self-consistency of a finite model set \(\mathcal{P}=\{h_1,\dots,h_K\}\) for an instance $\mathbf{x}\in\mathcal{S}$ as~\footnote{We slightly abuse notation and write \(\mathbf{x}\sim\mathcal{S}\) instead of $(\mathbf{x},y)$ to denote a feature vector sampled uniformly from a dataset $\mathcal{S}$.}:
\[
SC_{\mathcal{P}}(\mathbf{x})
=
\frac{1}{\binom{K}{2}}
\sum_{i<j} \mathbf{1}[h_i(\mathbf{x})=h_j(\mathbf{x})].
\]
$SC_{\mathcal{P}}(\mathbf{x})$ can be interpreted as the probability that two models drawn uniformly at random from $\mathcal{P}$ produce the same prediction for a given instance $\mathbf{x}$. 

Next, we provide a tight lower bound for expected self-consistency for any finite set of models $\mathcal{P}$ over a dataset $\mathcal{S}$. 

\begin{proposition}\label{prop:TLB}
Let \(\mathcal{P} =\{h_1,\ldots,h_K\}\) be a finite set of binary classifiers, and let
$\bar{L}_{\mathcal{S}}(\mathcal{P})=\frac{1}{K}\sum_{h\in\mathcal{P}}L_{\mathcal{S}}(h)$ denote the average 0--1 loss of models in \(\mathcal{P}\) on dataset $\mathcal{S}$. 

If $\mu=K\bar{L}_\mathcal{S}(P)$ and $\delta=\mu-\lfloor\mu\rfloor$, then the following tight lower bound holds for expected self-consistency:
\[
\mathbb{E}_{x\sim \mathcal{S}}[SC_\mathcal{P}(x)]
\ge
1-\frac{1}{\binom{K}{2}}
\left[
\mu(K-\mu)-\delta(1-\delta)
\right]
\]
\end{proposition}

The intuition behind the bound is as follows. If $r_\mathbf{x}$ denotes the number of models in $\mathcal{P}$ that misclassify $\mathbf{x}$, there are exactly $r_\mathbf{x}(K-r_\mathbf{x})$ pairs of models that disagree in their prediction for $\mathbf{x}$. Averaging this quantity over all instances connects pairwise disagreement to how the errors of models in $\mathcal{P}$ are distributed across $\mathcal{S}$. 

Here, $\mu$ is the average number of models that misclassify an instance $\mathbf{x}$, whereas $\delta$ is simply the fractional part of this average. The bound is tight when errors are spread as evenly as possible across instances i.e., when a fraction \(1-\delta\) of instances are misclassified by \(\lfloor\mu\rfloor\) models, and the remaining \(\delta\) fraction of instances are misclassified by \(\lfloor\mu\rfloor+1\) models. If errors are less evenly spread, more models would misclassify the same instances, thereby increasing the average self-consistency. This result also generalizes a related disagreement bound by~\citet{multiplicityFAccT} for the special case where all models in $\mathcal{P}$ have the same empirical risk. The proof of this bound is included in Appendix~\S~\ref{supp:self-consistency-bound}.

It is worth noting that Proposition~\ref{prop:TLB} does not depend on the model pool being a subset of the Rashomon Set, and therefore applies to any finite set of models. Conversely, membership in a Rashomon Set instantiates the bound a priori. Since every model in $\mathcal{R}(\mathcal{H},\epsilon,\mathcal{S})$ has empirical risk at most $L_{\mathcal{S}}(h_\mathcal{S}) + \epsilon$, the average loss of any pool drawn from it is capped by the same quantity. Substituting this bound for $\bar{L}_{\mathcal{S}}(\mathcal{P})$ therefore yields a lower bound on expected self-consistency that holds for every pool of $K$ models within the $\epsilon$-Rashomon Set, before any models are generated. We use Proposition~\ref{prop:TLB} to compare the predictive agreement in the model pools returned by \algoname{FairSLIM} against its worst-case guarantee. 

\subsubsection{Policies for Resolving Predictive Multiplicity.}
Finally, we evaluate three different policies for resolving predictive multiplicity: ensembling, random model selection (from $\mathcal{P}$ for every instance $\mathbf{x}$), and a \emph{lowest-risk policy}. Since \algoname{\RisCanvi} is used as a decision support system, we define the \emph{lowest-risk policy} over the predicted risk. To each instance \(\mathbf{x}\), it assigns the lowest possible risk score, i.e., $\min_{h \in \mathcal{P}} \lambda_h^{\top}\mathbf{x}$. In the binary classification setting, this is equivalent to assigning the non-recidivist label to an individual if at least one model in \(\mathcal{P}\) predicts that this individual will not recidivate. For the sake of completeness, we also compare it against a \emph{highest-risk policy} that assigns to an inmate the highest possible risk score, i.e., $\max_{h \in \mathcal{P}} \lambda_h^{\top}\mathbf{x}$.
\begin{figure}[t!]
    \centering
    \includegraphics[width=0.89\linewidth]{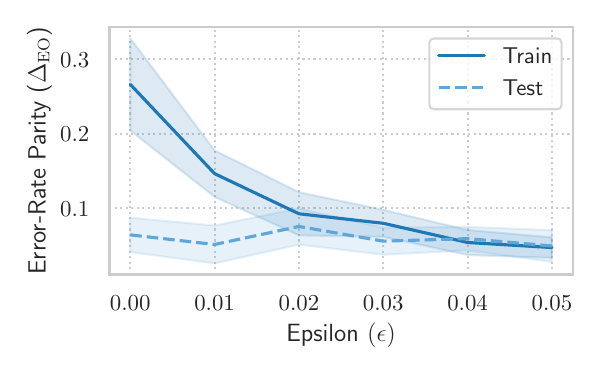}
    \caption{Error-rate parity ($\Delta_{\mathrm{EO}}$) achieved by the least discriminatory model learned using \algoname{FairSLIM} on the train and test sets for varying values of $\epsilon$. A tolerance of $\epsilon\leq .01$ is sufficient to create a model pool with reduced error-rate disparities ($\le5\%$) on the test set.}
    \label{fig:EOEPS}
\end{figure}
\section{Results}
\label{sec:results}

In this section, we present our results, evaluating models based on their predictive utility and error-rate disparities (\S~\ref{subsec:perfresults}), followed by a detailed analysis of predictive multiplicity and self-consistency (\S~\ref{subsec:predmulresults}).

\subsection{Predictive Performance and Equalized Odds}\label{subsec:perfresults}

Table~\ref{tab:performance_train_test} provides predictive performance and algorithmic fairness metrics for all models tested.
First, we note that all models trained on post-release outcomes labeled by our algorithm substantially outperform \algoname{\RisCanvi}.
While the accuracy of \algoname{\RisCanvi} may seem comparable to that of alternatives, this is largely driven by it predicting the majority class. As a result, it has poor balanced accuracy and a high false negative rate.
In contrast, both \algoname{SLIM} and \algoname{CatBoost} achieve balanced accuracies of approximately $70\%$, reducing the false negative rate by more than $65\%$.
Second, we find that \algoname{SLIM} matches \algoname{CatBoost}'s predictive performance on the test set, while exhibiting lower error-rate disparities across subgroups. In fact, searching for less discriminatory models (LDMs) using \algoname{FairSLIM} within $1\%$ of the best \algoname{SLIM} objective ($\epsilon \leq .01$) reduces the maximum error-rate disparity between any two subgroups to 5\% on the test set, as shown in Figure~\ref{fig:EOEPS}. Together, these results undermine the case for replacing \algoname{\RisCanvi} with a more complex model such as \algoname{CatBoost}, rather than with a simpler and interpretable scoring system like \algoname{SLIM}.
\begin{figure}[t!]
    \centering
    \includegraphics[width=\linewidth]{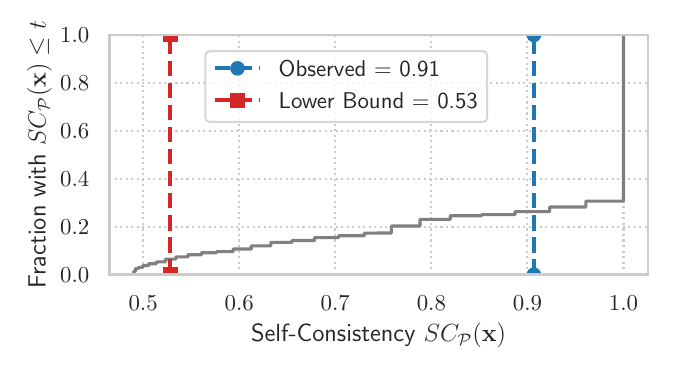}
    \caption{Empirical CDF of Self-Consistency on the test set for the aggregated model pool $\mathcal{P}$ with \mbox{$|\mathcal{P}| = 78, \epsilon\leq .01$}. Nearly $70\%$ instances exhibit perfect predictive agreement across all models, resulting in high average self-consistency.}
    \label{fig:SCCDF}
\end{figure}
\subsection{Predictive Multiplicity in Model Pools}\label{subsec:predmulresults}
We now study predictive multiplicity in the model pools returned by \algoname{FairSLIM}. In particular, we examine where and how often these models disagree on individual predictions, and evaluate a simple way of addressing the resulting predictive arbitrariness.

\subsubsection{Self-Consistency and Structural Diversity.} Despite the existence of many similarly accurate models with similar error-rate disparities across groups, model multiplicity translates into limited predictive multiplicity.
Instead, the model pools returned by \algoname{FairSLIM} exhibit high predictive agreement.
\begin{figure}[t!]
    \centering
    \includegraphics[width=\linewidth]{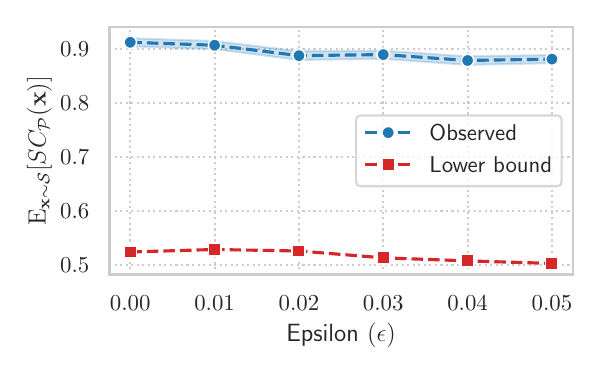}
    \caption{Observed versus worst-case expected self-consistency on the test set for varying values of $\epsilon$. While the lower bound becomes increasingly pessimistic as $\epsilon$ grows, the observed average self-consistency remains high.}
    \label{fig:ASCEpsTest}
\end{figure}
For instance, as shown in Figure~\ref{fig:SCCDF}, the probability that two models sampled uniformly at random from the pool assign the same prediction to a random individual from the test set is 91\%. Moreover, Figure~\ref{fig:ASCEpsTest} shows that the average self-consistency remains substantially higher across values of $\epsilon$ than worst case analysis suggests (based on the lower bound from Proposition~\ref{prop:TLB}). This trend also holds at a subgroup-level, and is stable across sampled pool sizes $k \in \{5, 10, 15, 20, 25\}$ for all values of $\epsilon$, indicating that our conclusions do not depend on the size or composition of the aggregated pool (see Appendix~\S~\ref{supp:self-consistency-subgroups}--\ref{supp:self-consistency-poolsize}).

Disagreement in prediction is limited to instances that lie close to the decision boundary on average, as shown in Figure~\ref{fig:absmargin01}. We measure closeness in terms of the average absolute margin across \(\mathcal{P}\), given by \(\mbox{\(\frac{1}{K}\sum_{h \in \mathcal{P}} |\lambda_h^\top \mathbf{x}-\gamma|\)}\). One may think this is because the models in the pool are structurally alike, with similar feature coefficients and therefore similar risk scores. In reality, the model pool is structurally diverse, with coefficients varying significantly across both static and dynamic features (see Figure~\ref{fig:coefvar}). Overall, we find that models largely agree in their predictions despite relying on different scoring rules, indicating that high self-consistency can coexist with structural diversity.

This finding is not a consequence of model multiplicity alone: models with comparable error rates can in principle disagree on a large fraction of instances (cf. Proposition~\ref{prop:TLB}). Furthermore, prior empirical work frequently reports substantial predictive multiplicity on benchmark datasets, including \algoname{COMPAS}~\cite{pmlr-v119-marx20a, arbitraryAAAI}. In our setting, predictive multiplicity is limited because the errors of different models largely overlap, so that disagreement is confined to the small set of low-margin instances identified above. One contributing factor may be that all models in our pools satisfy monotonicity constraints and exhibit comparably low error-rate disparities. These additional requirements further restrict the set of admissible models, whereas the pools studied in prior work are constrained only by empirical risk.

\begin{figure}[t!]
    \centering
    \includegraphics[width=\linewidth]{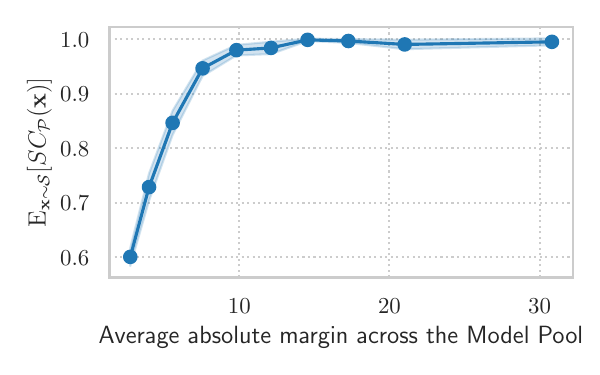}
    \caption{Expected self-consistency as a function of average absolute margin across the aggregated model pool $\mathcal{P}$ for \mbox{\(\epsilon\le.01\)}. Predictive disagreement is concentrated among low-margin instances, while high-margin instances have almost perfectly consistent predictions.}
    \label{fig:absmargin01}
\end{figure}
\begin{figure*}[h!]
\centering
\includegraphics[width=\linewidth]{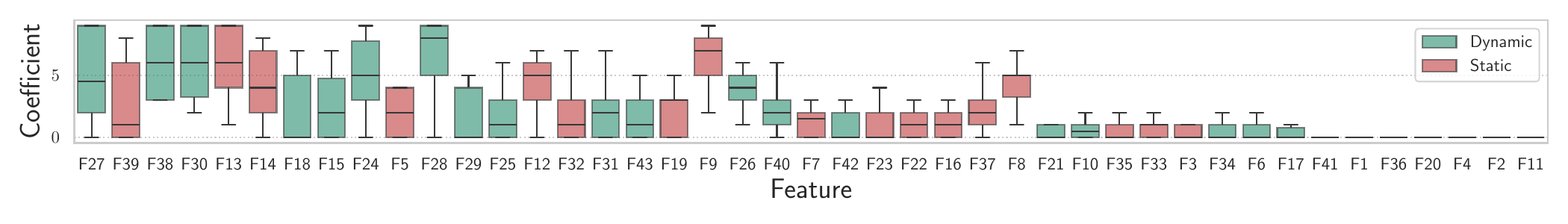}
\caption{Variance in feature coefficients across the model pool for $\epsilon \leq .01$. Although $\mathcal{P}$ exhibits high average self-consistency, coefficients vary substantially across both dynamic and static features. Features are sorted by variance in decreasing order.}
\label{fig:coefvar}
\end{figure*}
\begin{figure*}[t!]
\centering
\begin{subfigure}{0.49\textwidth}
    \centering
    \includegraphics[width=\textwidth]{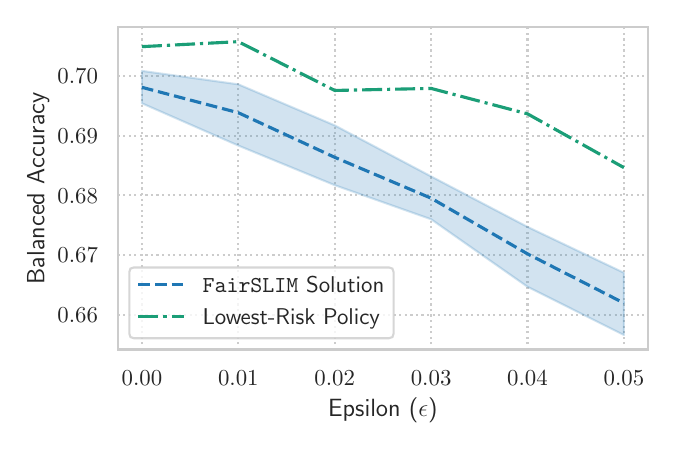}
    \caption{Balanced accuracy of the best \algoname{FairSLIM} solution and the \emph{lowest-risk policy} over the model pool $\mathcal{P}$ for varying values of $\epsilon$.}
    \label{fig:lrp}
\end{subfigure}
\hfill
\begin{subfigure}{0.49\textwidth}
    \centering
    \includegraphics[width=\textwidth]{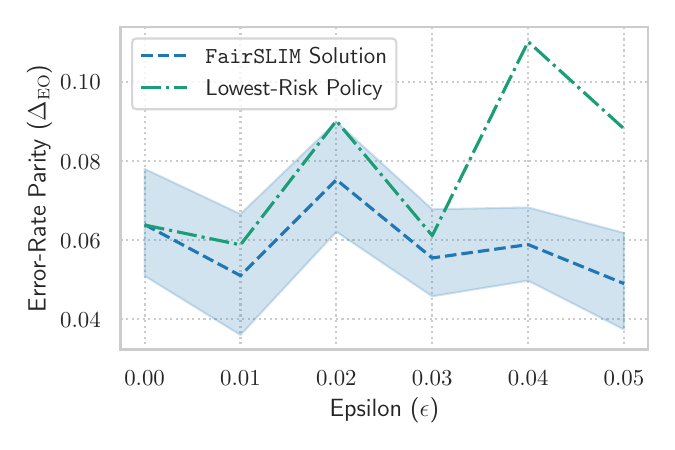}
    \caption{Error-rate parity ($\Delta_{\mathrm{EO}}$) of the best \algoname{FairSLIM} solution and the \emph{lowest-risk policy} over the model pool $\mathcal{P}$ for varying values of $\epsilon$.}
    \label{fig:lrpeo}
\end{subfigure}
\caption{A \emph{lowest-risk policy} over the aggregated model pool $\mathcal{P}$ achieves higher balanced accuracy than the best \algoname{FairSLIM} solution (lowest $\Delta_{\mathrm{EO}}$) on the test set, while demonstrating comparable error-rate parity ($\Delta_{\mathrm{EO}}$) for $\epsilon \leq .03$.}
\label{fig:LRP}
\end{figure*}
\subsubsection{A Lowest-Risk Policy to Address Predictive Multiplicity.}
Predictive multiplicity in the model pools returned by \algoname{FairSLIM} is concentrated among approximately $30\%$ of instances on the test set, most of which lie close to the decision boundary on average.
Therefore, any mechanism for resolving the resulting arbitrariness in risk assessment will predominantly affect instances for which models cannot produce a confident prediction.
In this case, we suggest applying the \emph{favor rei} principle (``rule in favor of the defendant'' [when in doubt]) and the anti-subordination principle, preferring decision-making policies that actively address systemic biases~\cite{keswani2024algorithmic}. 
Specifically, we propose assigning each inmate the risk computed by the model in the pool that produces the lowest risk score for that inmate.
This simple policy, which we call \algoname{FairSLIM-LRP} (LRP stands for lowest risk policy) achieves higher accuracy and lower false positive rate disparity than both \algoname{CatBoost} and \algoname{SLIM} on the test set (see Table~\ref{tab:performance_train_test}). In terms of relaxed equalized odds, \algoname{FairSLIM-LRP} performs just as well as the best \algoname{FairSLIM} solution, i.e., with the lowest $\Delta_{\mathrm{EO}}$, while achieving higher balanced accuracy, as shown in Figure~\ref{fig:LRP}.

The lowest-risk policy also achieves the highest balanced accuracy among the resolution mechanisms we consider, including ensembling, random model selection, and a highest-risk policy (which we test for the sake of completeness), while maintaining comparable error-rate disparities for $\epsilon\le.03$. Figure~\ref{fig:policycomparison} provides a more detailed comparison with various values of $\epsilon$. The effectiveness of \algoname{FairSLIM-LRP} can be explained by two factors. First, low risk is the majority outcome: most individuals do not recidivate, and the policy intervenes on a relatively small number of cases. Second, our training data does not contain non-recidivist labels for the set of recently freed people who have been outside of prison for more than one year and less than five years. The majority of them will not recidivate, as most recidivists reoffend within the first year of being released. However, the second factor reflects the reality of recidivism prediction rather than a limitation of our dataset. Any training set assembled at the time of deployment lacks non-recidivist labels among recent releases, since non-recidivism is only established once the five-year observation window has elapsed.

\begin{figure*}[t!]
    \centering
    \includegraphics[width=0.9\linewidth]{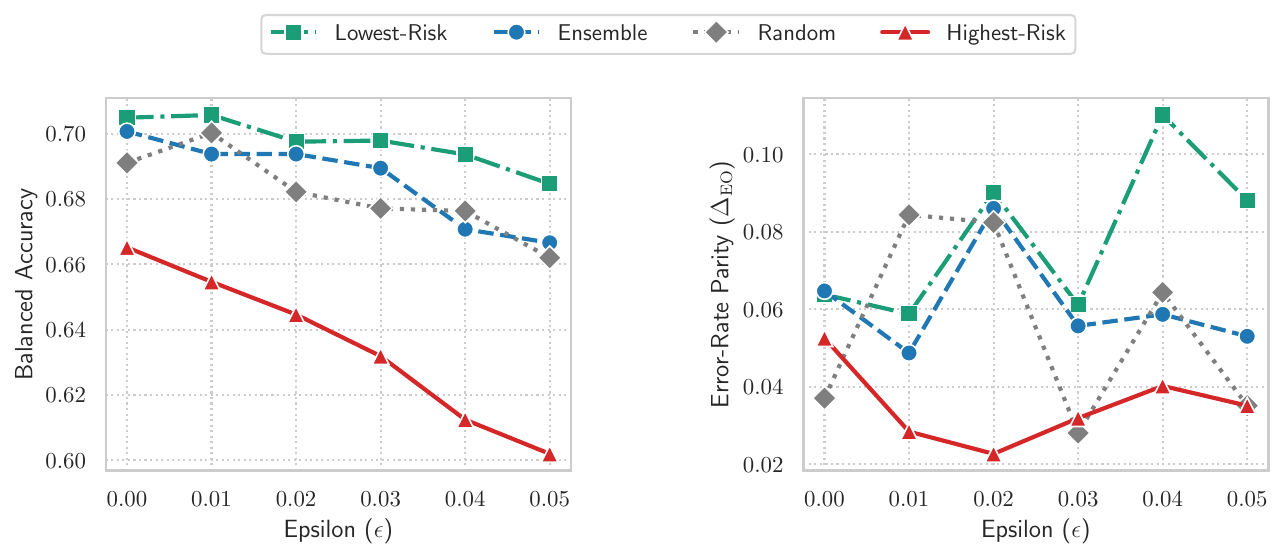}
    \caption{Comparison of policies for resolving predictive multiplicity over the aggregated model pool \(\mathcal{P}\). The \emph{lowest-risk policy} yields the highest balanced accuracy among those considered and maintains comparable error-rate parity for \(\epsilon\leq .03\).}
    \label{fig:policycomparison}
\end{figure*}
\section{Discussion}
\label{sec:discussion}
Our work addresses several challenges in \algoname{\RisCanvi}, a system that supports recidivism risk assessment for thousands of incarcerated people every year. We work closely with researchers in sentence enforcement from the Department of Justice of \Catalonia to translate legal rules into an algorithm that systematically labels post-release outcomes. This helps us obtain \textbf{reliable training data} and eliminates reliance on a time-consuming labeling process that is susceptible to human error. In practice, manual labeling of a test set would be recommended for periodic testing of the system, but not needed to generate a much larger set required for training. Using this data, we extend existing MILPs to learn interpretable models that are significantly more accurate, distribute predictive errors more evenly across groups, and ensure by design that rehabilitative progress lowers risk scores.

Our analysis of \textbf{predictive multiplicity} in the model pool highlights several important considerations. The existence of structurally diverse models does not necessarily mean vastly different predictions. Empirically, models exhibit much higher predictive agreement than what the theoretical lower bounds guarantee. Thus, it is entirely possible that the \textbf{arbitrariness} due to predictive multiplicity only affects a small proportion of instances. In our case, these instances lie close to the decision boundary on average, making their predictions ambiguous across most models. Our results indicate that a simple policy following the \emph{favor rei} principle that prevents arbitrariness by assigning each inmate the lowest possible risk among the models is effective. Below, we situate this policy within its institutional and legal context, and discuss the trade-offs it entails.

\spara{On the Lowest-Risk Policy for Punitive Interventions.}
Beyond predictive performance and lower error-rate disparities, the case for a lowest-risk policy rests on the \textit{punitive} (as opposed to \textit{assistive}) nature of recidivism risk assessments, where the cost of over-predicting risk is borne by the individual. The law governing these assessments itself acknowledges their uncertainty. In particular, Article~67 of the Spanish General Penitentiary Law~\cite{LOGP1979} requires, upon completion of rehabilitative treatment or as release approaches, a final prognosis report containing a ``probability judgment'' on the individual's future behavior, which is considered in decisions on conditional release.
Since the law itself demands a judgment of probability, residual doubt is inherent to the assessment, and the \emph{favor rei} principle extends naturally to resolving it in favor of the inmate. In the context of the anti-subordination principle~\cite{keswani2024algorithmic}, this policy helps address structural asymmetries created by predictive multiplicity. In actuarial risk assessment, the scoring rule underlying a model is typically not visible to incarcerated individuals, or the prison staff responsible for evaluating them. As a result, affected individuals lack the resources and information needed to contest assigned risk scores, even when similarly valid models could have produced less punitive outcomes. This concern is especially consequential in discrimination claims, where the burden of producing \emph{prima facie} evidence rests with the plaintiff under EU law~\cite{EUdiscrimlaw}. 
A lowest-risk policy internalizes this burden, functioning as a built-in appeal that automatically grants what a successful objection would otherwise obtain. Moreover, it counterbalances the documented tendency of both laypeople and experts to overestimate recidivism risk~\cite{Portela2025}. However, this policy entails a public-safety trade-off, since resolving arbitrariness in favor of inmates increases false negatives (albeit far fewer than the deployed model). It also means evaluating different individuals with different models, though unlike ensembling or randomization, each assessment remains attributable to a single interpretable scoring system. While our analysis concerns recidivism, we believe that our case for resolving predictive arbitrariness at the institution's expense rather than the individual's, deserves consideration for interventions of a \textit{punitive} nature more broadly.

\spara{Periodic Evaluations and Model Updates.}
Once risk assessment is viewed as an ongoing institutional process rather than a one-time prediction task, the considerations outlined above become even more complex. Articles 17 and 72 of the AI Act require periodic evaluations of high-risk systems ~\cite{ai_act_2024}, which may necessitate re-training of models on more recent data. In this setting, predictive arbitrariness becomes more complex because model pools can change over time. For instance, if an inmate is assessed by different models in consecutive evaluations, their risk assessments may fail to reflect rehabilitative progress if the updated model omits or downweights certain dynamic features. This prevents \textbf{algorithmic recourse}, i.e., the provision of information on dynamic (mutable) features whose changes are likely to reduce the computed risk~\cite{algrecourse}.
However, we believe that rehabilitation programs should not be driven by \algoname{\RisCanvi}. Instead, they should be guided by models specifically designed to identify interventions that are most likely to reduce recidivism risk---not models that predict risk in the absence of interventions~\cite{interoverpred}.
This is a separate and rather challenging task in the current institutional setting, in which, for instance, the policy of evaluating each inmate every six months is not strictly implemented, as reported in \S~\ref{subsec:riscanvi}, and in which rehabilitative progress of inmates is not guaranteed to lower their risk scores. Together, these concerns motivate our focus on models with better predictive performance, lower error-rate disparities, and explicit mechanisms for risk-reduction.
\section{Conclusions}\label{sec:conclusions}
%
This work, first, highlights the importance of automating manual labeling procedures, as opposed to relying on proxy labels or outdated data. More specifically, the much larger set of post-release outcomes labeled by our algorithm materially improves model development by increasing predictive performance, reducing error-rate disparities and aligning risk scores with domain expertise. 

Second, our results show that predictive multiplicity is not solely determined by the number of sufficiently distinct and similarly performant models; what matters instead is the extent to which their predictions overlap.
Consequently, even when many such models may exist, predictive multiplicity may be far less severe than worst-case analysis suggests. 
Proposition~\ref{prop:TLB} serves as a useful reference for empirical estimates of predictive agreement. For instance, using procedures that independently sample from the Rashomon Set, one can estimate expected self-consistency and compare it against the worst-case guarantee.
Similarly, estimating self-consistency over models that satisfy additional desiderata (statistical non-discrimination criteria, monotonicity etc.) can reveal the residual predictive arbitrariness after imposing these constraints.

A broader implication of this research is that, in the presence of model multiplicity, \emph{where} predictive disagreement occurs can greatly inform mechanisms for addressing the resulting arbitrariness in decision-making.
In particular, the geometry of disagreement (e.g., where it is concentrated relative to the decision threshold) may support policies that are grounded in existing institutional principles (e.g., resolving ambiguity in favor of subjects). 

\subsection{Limitations and Future Work}
Our work has several limitations stemming from the institutional setting, resource availability, and the scope of our analysis.
First, our analysis is based on a single operational system in \Catalonia. Therefore, our empirical results are highly contextual, and may not generalize to other jurisdictions with different legal rules, risk assessment systems, and data collection practices.
While we cannot release the dataset, it can be requested from the \emph{Àrea de planificació i Projectes Estratègics} of the Department of Justice of Catalonia, e-mail: planificacio.dj@gencat.cat -- as we did for this research.
Second, due to limited data, our work does not explicitly address the challenges introduced by periodic risk assessments and model updates.
Third, our analysis is limited to the predictive aspect of risk assessment, whereas \algoname{\RisCanvi} operates as a decision support system.
An end-to-end evaluation of the effectiveness of any new model or policy is necessary, and should be conducted in a decision support setting. This means evaluating the extent to which actual users of the system understand and interpret risk scores by different models, and more importantly, whether they reach more accurate decisions in a more reliable manner using this tool. 

Fourth, our study on predictive multiplicity is focused on the hypothesis class of sparse integer linear models. This choice reflects the design of \algoname{\RisCanvi}, and also the broader use of linear scoring systems in recidivism risk assessment, where static and dynamic features contribute additively to the final risk score. Generalizing some of the results, e.g., those related to the distribution of predictive disagreement with respect to the decision boundary, requires the exploration of a broader class of models.

Finally, the self-consistency we report in our work is \emph{procedural}, i.e., a property of our training regime, whereas Proposition~\ref{prop:TLB} bounds the \emph{worst case} for any finite set of models. Establishing \emph{existential} multiplicity, i.e., whether a monotone LDM in the Rashomon Set can flip an individual's prediction, requires instance-specific certification rather than pool-based estimation, and is left for future work.

\section*{Acknowledgments}
This work was made possible by a long-standing collaboration with the Centre d'Estudis Jurídics i Formació Especialitzada (CEJFE). We are grateful for the support of CEJFE, especially Marian, Alba, Marta, Susana, and Abril, for sharing their expertise on sentence enforcement. We also disclose that CC, through UPF, provided training to CEJFE in 2021 on causal analysis methods.
AS is supported by the Vienna Science and Technology Fund (WWTF) [Grant ID: 10.47379/VRG23013]. 
Comments by AIES reviewers also contributed to strengthening the presentation of this paper.
\clearpage

\section{Researcher Positionality Statement}
We now reflect on our position and our engagement with \algoname{\RisCanvi} as an algorithmic institution~\cite{mendoncca2024algorithmic}.
We start by acknowledging our formal training as computer scientists, with additional background in social and responsible computing. Despite our due diligence, it is possible that we overlook important considerations by researchers from other disciplines who work more closely with risk assessments in criminal justice contexts. Nevertheless, we collaborate closely with domain experts in sentence enforcement. These experts have the responsibility of generating official statistics on recidivism in \Catalonia, and have years of experience in various research projects on the topic.
Lastly, our lived experiences are far removed from those of incarcerated individuals who undergo routine evaluations by \algoname{\RisCanvi}. This significantly limits our ability to assess the efficacy of institutional processes through which rehabilitative needs are identified and met.

\section{Adverse Impacts Statement}
First and foremost, our work does not engage with the abolitionist debate regarding prisons. Rather, our work is situated within the existing legal obligation in \Catalonia to conduct risk assessments before the definitive release of any person. Although this position reinforces existing carceral structures, we believe it is urgent to address systemic issues in recidivism risk assessment tools because they continue to shape the lives of incarcerated people.

Second, \algoname{\RisCanvi} is based on the RNR (\emph{risk-need-responsivity}) model from criminology~\cite{RNR}, which requires rehabilitative interventions to (i) match their intensity to the level of recidivism \textbf{risk}, (ii) address individual criminogenic \textbf{needs}, and (iii) be tailored to individual \textbf{responsivity} to treatment. \algoname{\RisCanvi} operationalizes risk and responsivity through actuarial assessments, but provides limited visibility into how rehabilitative needs are addressed in the day-to-day lives of incarcerated people. As a result, systems like \algoname{\RisCanvi} are susceptible to reducing their rehabilitation to what can be measured through actuarial assessments.

Third, \algoname{\RisCanvi} operates as a decision support system, in compliance with Article~22 of the GDPR that affords EU residents the right to not be subject to solely automated decisions that produce significant effects~\cite{EU2016GDPRArticle22}. Recent work supports this operational setting, showing that recidivism risk assessments are most accurate when they follow a structured approach while trained professionals retain the agency to overrule model predictions~\cite{Portela2025}. Therefore, we advocate for using such tools solely as decision support systems, operated by trained professionals under appropriate institutional incentives that promote accountability and care. Nothing in this paper should be interpreted as a justification to reduce careful scrutiny of the outputs of these systems, given that they are only as reliable as their weakest link. For instance, investigative reporting on other risk assessment tools has documented cases where record keeping errors contributed to parole denials despite substantial rehabilitative progress~\cite{Wexler2017CodeOfSilence}. 

Fourth, \algoname{\RisCanvi} must be used within its intended scope. However, there have been reports of multiple cases involving a high risk score computed by \algoname{\RisCanvi}, which has been reduced, with justification, by human evaluation teams, but then prosecutors have cited the tool's score in parole proceedings as grounds for opposing release~\citecourtuse.
Nothing in our research should be interpreted as justifying out-of-scope uses of \algoname{\RisCanvi}.

\clearpage

\bibliography{aaai2026}

\clearpage
\appendix
\section{\RisCanvi Features}
\label{supp:riscanvi-features}

Table~\ref{tab:riscanvifeatures} at the end of this document, lists all of the static and dynamic features used by \algoname{\RisCanvi} for risk assessment.
These were selected in 2009 from a pool of over 100 features, through a process that tested their statistical dependency with relevant outcomes (including recidivism) plus the input of domain experts~\citeriscanvi.

\section{Rule-Based Recidivism Labeling}
\label{supp:labelingalgo}

Below, we provide pseudocode of our recidivism labeling algorithm. Note that more fine-grained details are abstracted as the algorithm is highly customized to the internal database. Put simply, it searches for crimes committed by an inmate in less than five years \textbf{after} being released for which the inmate \textbf{re-enters prison} to serve a sentence.

\begin{algorithm}[h]
\caption{Rule-based labeling of penal recidivism}
\label{alg:recidivism-labeling}
\textbf{Input}: Inmate \(I\), release date \(R\), release procedure \(P\)\\
\textbf{Output}: $\mathrm{True} / \mathrm{False}$
\begin{algorithmic}[1]
\STATE Let \(t^\star \gets \infty\)
\STATE $D \gets 5$ years
\FORALL{cases \(C \in \mathrm{Cases}(I)\)}
    \IF{\(\mathrm{ProcedureID}(C)=P\)}
        \STATE \textbf{continue} \hfill\COMMENT{Case is related to the current release.}
    \ELSIF{\(\mathrm{Verdict}(C)=\mathrm{Conviction}\)}
        \STATE \(M \gets \{m\in\mathrm{Crimes}(C): \mathrm{Date}(m)>R\}\)
        \STATE \(d \gets \min_{m\in M}\mathrm{Date}(m)\)
    \ELSIF{\(\mathrm{Verdict}(C)=\mathrm{Preventive\;Sentencing}\)}
        \STATE \(d \gets \mathrm{DetentionDate}(C)\)
    \ELSE
        \STATE \textbf{continue}
    \ENDIF
    \IF{\(d>R\)}
        \STATE \(t^\star \gets \min\{t^\star,\mathrm{TimeDiff}(R,d)\}\)
    \ENDIF
\ENDFOR
\STATE \textbf{return} \(t^\star < D\)
\end{algorithmic}
\end{algorithm}

\section{FairSLIM Extension to $m>2$ Groups}
\label{supp:fairslim-milp}
\subsubsection{Notation.}
We use $[n]$ to denote the set $\{1,\dots,n\}$, and the function $I:\mathcal{S}\to[n]$ maps instances in $\mathcal{S}$ to their indices. We define $\mathcal{S}_G =\{\mathbf{x}\in\mathcal{S}\mid g(\mathbf{x})=G\}$. $\mathcal{L}_j$ denotes the set of values that the coefficient of feature $j$, i.e., $\lambda_j$ can take.\\

The complete \algoname{FairSLIM} MILP is as follows:

{\small\begin{align}
\min_{\lambda}\quad 
& \Delta_{\mathrm{EO}} \tag{1}\\
\text{s.t.}\quad
& \lambda_j = \sum_{\omega \in \mathcal{L}_j} \omega\cdot u_{j\omega}
&& \forall j\in[p] \tag{2}\\
& \sum_{\omega \in \mathcal{L}_j} u_{j\omega}=1
&& \forall j\in[p] \tag{3}\\
& \frac{1}{n}\sum_{i=1}^{n}z_i \leq L_{\mathcal{S}}(h_{\mathcal{S}})+\epsilon \tag{4}\\
& M_i z_i \geq \gamma - y_i\lambda^\top \mathbf{x}_i
&& \forall i\in[n] \tag{5}\\
& O_i(1-z_i) \geq y_i\lambda^\top \mathbf{x}_i-\gamma
&& \forall i\in[n] \tag{6}\\
& \mathrm{FPR}_G
=
\frac{1}{|\mathcal{S}_{G}^{-}|}
\sum_{i\in I(\mathcal{S}_{G}^{-})} z_i
&& \forall G\in\mathcal{G} \tag{7}\\
& \mathrm{FNR}_G
=
\frac{1}{|\mathcal{S}_{G}^{+}|}
\sum_{i\in I(\mathcal{S}_{G}^{+})} z_i
&& \forall G\in\mathcal{G} \tag{8}\\
& \mathrm{FPR}_{\min}\leq \mathrm{FPR}_G \leq \mathrm{FPR}_{\max}
&& \forall G\in\mathcal{G} \tag{9}\\
& \mathrm{FNR}_{\min}\leq \mathrm{FNR}_G \leq \mathrm{FNR}_{\max}
&& \forall G\in\mathcal{G} \tag{10}\\
& \Delta_{\mathrm{EO}} \geq \mathrm{FPR}_{\max}-\mathrm{FPR}_{\min} \tag{11}\\
& \Delta_{\mathrm{EO}} \geq \mathrm{FNR}_{\max}-\mathrm{FNR}_{\min} \tag{12}\\
& u_{j\omega} \in \{0,1\} && \forall j \in [p], \omega \in \mathcal{L}_j \tag{13}\\
& z_i \in \{0,1\} && \forall i \in [n] \tag{14}
\end{align}}
\;
\subsubsection{MILP Description.} Equations (2) and (3) ensure that every coefficient $\lambda_j$ only takes one value from $\mathcal{L}_j$ by using $u_{j\omega}$ as selector variables. Equation (4) represents the Rashomon constraint which ensures that the balanced 0--1 loss of \algoname{FairSLIM} stays within $\epsilon$ tolerance of $L_{\mathcal{S}}(h_\mathcal{S})$ where $h_\mathcal{S}$ is the incumbent feasible solution obtained from \algoname{SLIM}. Here, $z_i = 1$ if instance $\mathbf{x}_i$ has been misclassified and $0$ otherwise. This is enforced by margin constraints (5) and (6). In (5), $M_i = \max_{\lambda \in \mathcal{L}} (\gamma - y_i\lambda^\top \mathbf{x}_i)$ is the maximum margin violation (during misclassification). Similarly, in (6), $O_i = \max_{\lambda \in \mathcal{L}} (y_i\lambda^\top \mathbf{x}_i-\gamma)$ is an upper bound on the margin (when points are correctly classified). Constraints (7) and (8) compute false positive and false negative rates respectively for all groups $G \in \mathcal{G}$. In (9) and (10), auxiliary variables $\mathrm{FPR}_{\max},\mathrm{FPR}_{\min},\mathrm{FNR}_{\max},\mathrm{FNR}_{\min}$ track the maximum and minimum false positive and false negative rates. Constraints (11) and (12) ensure that the objective $\Delta_\mathrm{EO}$ is lower bounded by the maximum error-rate disparity between any two groups $G_i,G_j \in \mathcal{G}$. Lastly, (13) and (14) simply define the domain of $z_i$ and $u_{j\omega}$. \textbf{To summarize, our contributions are constraints (7-12)}. 

\section{Convergence and Runtime}
\label{supp:convergeruntime}

Figures~\ref{fig:slimconvergence} and~\ref{fig:fairslim_convergence} show the convergence behavior of \algoname{SLIM} and \algoname{FairSLIM} respectively. While neither MILP reaches provable optimality within the allotted runtime, the incumbent objective stabilizes rapidly, with most improvements occurring within the first 30 minutes. Overall, we observe that the optimality gap is much smaller for the \algoname{FairSLIM} runs.

\begin{figure}[h]
    \centering
    \includegraphics[width=\linewidth]{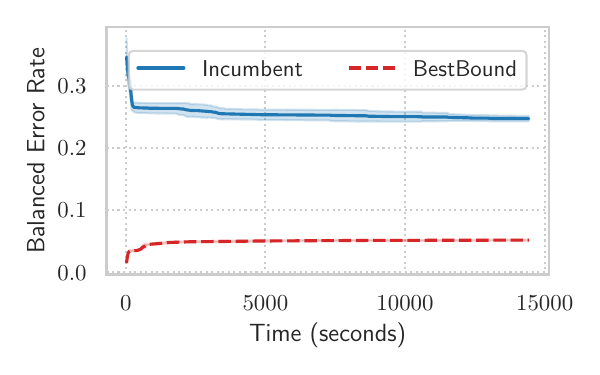}
    \caption{Convergence behavior of \algoname{SLIM}. Although the solver does not converge within the allotted time, the majority of improvement in the objective occurs within the first 30 minutes, with only marginal gains thereafter.}
    \label{fig:slimconvergence}
\end{figure}

\begin{figure*}[h]
    \centering\small

    \begin{subfigure}[b]{0.31\linewidth}
        \centering
        \includegraphics[width=\linewidth]{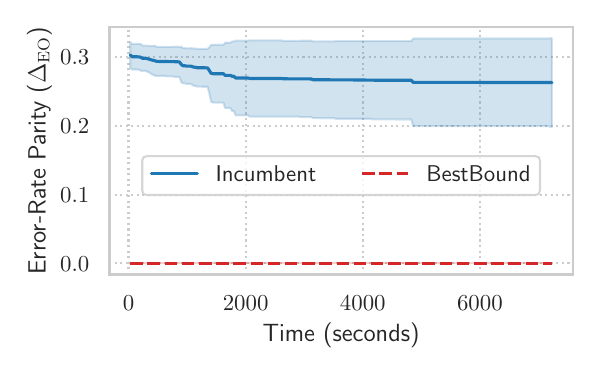}
        \caption{\(\epsilon = .0\)}
        \label{fig:conv0}
    \end{subfigure}
    \hfill
    \begin{subfigure}[b]{0.31\linewidth}
        \centering
        \includegraphics[width=\linewidth]{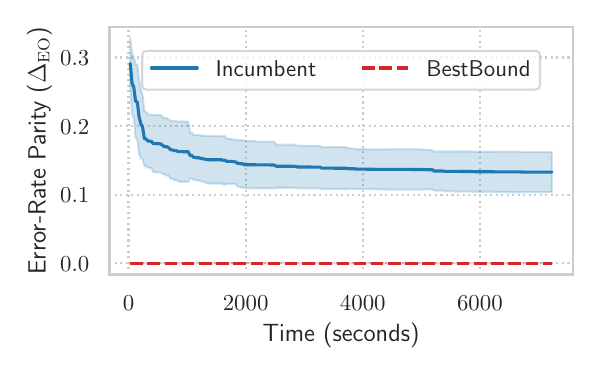}
        \caption{\(\epsilon = .01\)}
        \label{fig:conv001}
    \end{subfigure}
    \hfill
    \begin{subfigure}[b]{0.31\linewidth}
        \centering
        \includegraphics[width=\linewidth]{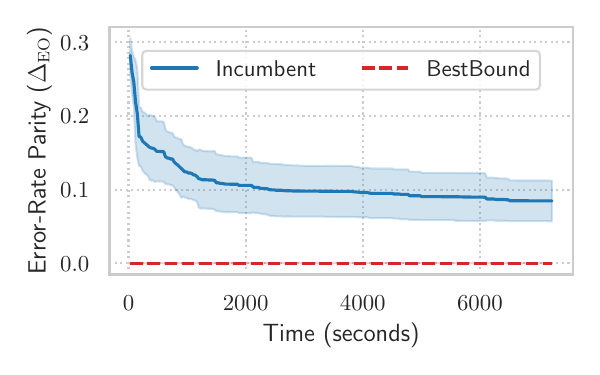}
        \caption{\(\epsilon = .02\)}
        \label{fig:conv002}
    \end{subfigure}

    \vspace{0.5em}

    \begin{subfigure}[b]{0.31\linewidth}
        \centering
        \includegraphics[width=\linewidth]{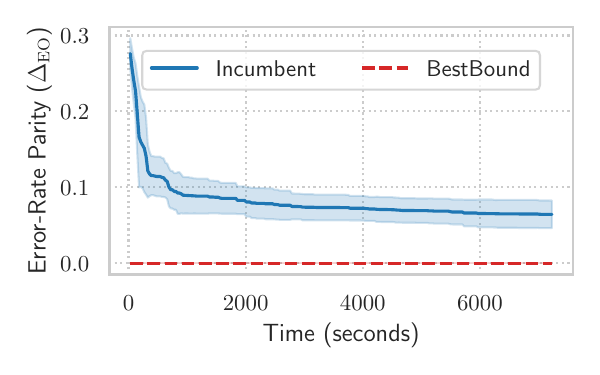}
        \caption{\(\epsilon = .03\)}
        \label{fig:conv003}
    \end{subfigure}
    \hfill
    \begin{subfigure}[b]{0.31\linewidth}
        \centering
        \includegraphics[width=\linewidth]{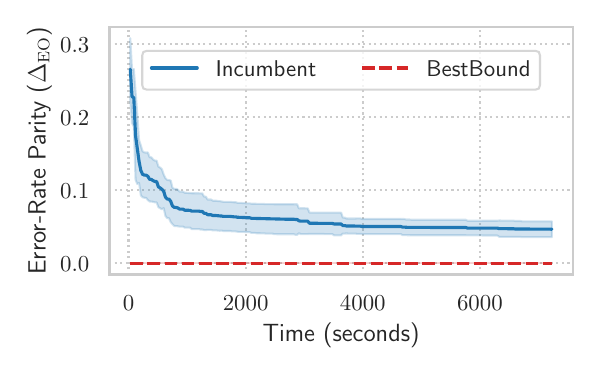}
        \caption{\(\epsilon = .04\)}
        \label{fig:conv004}
    \end{subfigure}
    \hfill
    \begin{subfigure}[b]{0.31\linewidth}
        \centering
        \includegraphics[width=\linewidth]{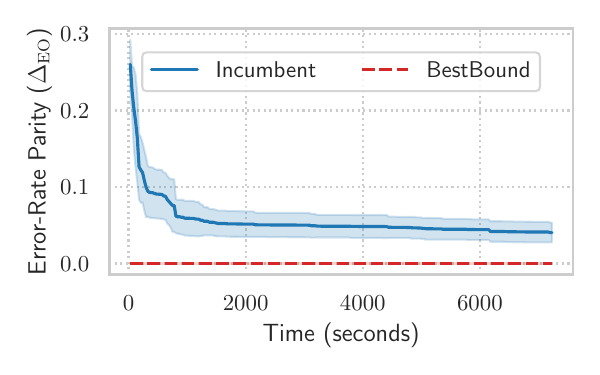}
        \caption{\(\epsilon = .05\)}
        \label{fig:conv005}
    \end{subfigure}

    \caption{Convergence behavior of FairSLIM across different Rashomon tolerances \(\epsilon\) on the train set. Although the solver does not converge within the allotted time, the majority of improvement in the objective occurs within the first 30 minutes.}
    \label{fig:fairslim_convergence}
\end{figure*}

\section{Bound for Expected Self-Consistency}
\label{supp:self-consistency-bound}
In this section, we restate our bound on expected self-consistency given in the paper, followed by its proof. For notational convenience, we encode predictions and labels as binary vectors in \(\{0,1\}^{|\mathcal{S}|}\). 

\begin{proposition}\label{supp:TLB}
Let \(\mathcal{P} =\{h_1,\ldots,h_K\}\) be a finite set of binary classifiers, and let
$\bar{L}_{\mathcal{S}}(\mathcal{P})=\frac{1}{K}\sum_{h\in\mathcal{P}}L_{\mathcal{S}}(h)$ denote the average 0--1 loss of models in \(\mathcal{P}\) on dataset $\mathcal{S}$. 

If $\mu=K\bar{L}_\mathcal{S}(P)$ and $\delta=\mu-\lfloor\mu\rfloor$, then the following tight lower bound holds for expected self-consistency:
\[
\mathbb{E}_{x\sim \mathcal{S}}[SC_\mathcal{P}(x)]
\ge
1-\frac{1}{\binom{K}{2}}
\left[
\mu(K-\mu)-\delta(1-\delta)
\right]
\]
\end{proposition}

\begin{proof}
For each instance \((\mathbf{x},y)\in\mathcal{S}\), let \(r_{\mathbf{x}} = \sum_{h\in\mathcal{P}}\mathbf{1}[h(\mathbf{x})\neq y]\) denote the number of models in \(\mathcal{P}\) that misclassify \(\mathbf{x}\). Then \(K-r_{\mathbf{x}}\) models classify \(\mathbf{x}\) correctly. The number of disagreeing model pairs on $\mathbf{x}$ is simply $r_\mathbf{x}(K-r_\mathbf{x})$. Hence, 
the self-consistency of \(\mathcal{P}\) on instance $\mathbf{x}$ can be written as:
\[
SC_{\mathcal{P}}(\mathbf{x})
=
1 - \frac{1}{\binom{K}{2}}r_\mathbf{x}(K-r_\mathbf{x})
\]
Taking expectations over \(\mathbf{x}\sim\mathcal{S}\), we get:
\begin{align}
\mathbb{E}_{\mathbf{x}\sim\mathcal{S}}[SC_{\mathcal{P}}(\mathbf{x})]
&=
1-\frac{1}{\binom{K}{2}}
\mathbb{E}_{\mathbf{x}\sim\mathcal{S}}
\left[
r_{\mathbf{x}}(K-r_{\mathbf{x}})
\right] \notag\\
&=
1-\frac{1}{\binom{K}{2}}
\left(
K\mathbb{E}_{\mathbf{x}\sim\mathcal{S}}[r_{\mathbf{x}}]
-
\mathbb{E}_{\mathbf{x}\sim\mathcal{S}}[r_{\mathbf{x}}^2]
\right)
\end{align}
Observe that
\begin{align}
\mathbb{E}_{\mathbf{x}\sim\mathcal{S}}[r_{\mathbf{x}}]
&=
\frac{1}{|\mathcal{S}|}
\sum_{(\mathbf{x},y)\in\mathcal{S}}
\sum_{h\in\mathcal{P}}
\mathbf{1}[h(\mathbf{x})\neq y] \notag\\
&=
\sum_{h\in\mathcal{P}}
\frac{1}{|\mathcal{S}|}
\sum_{(\mathbf{x},y)\in\mathcal{S}}
\mathbf{1}[h(\mathbf{x})\neq y] \notag\\
&=
\sum_{h\in\mathcal{P}} L_{\mathcal{S}}(h) \notag\\
&=
K\bar{L}_{\mathcal{S}}(\mathcal{P}).
\end{align}
Substituting (2) in (1), we get
\begin{align}
\mathbb{E}_{\mathbf{x}\sim\mathcal{S}}[SC_{\mathcal{P}}(\mathbf{x})]
&=
1-\frac{1}{\binom{K}{2}}
\left(
K^2\bar{L}_{\mathcal{S}}(\mathcal{P})
-
\mathbb{E}_{\mathbf{x}\sim\mathcal{S}}[r_{\mathbf{x}}^2]
\right)
\end{align}
To obtain a lower bound on $\mathbb{E}_{\mathbf{x}\sim\mathcal{S}}[SC_{\mathcal{P}}(\mathbf{x})]$, we must minimize $\mathbb{E}_{\mathbf{x}\sim\mathcal{S}}[r_{\mathbf{x}}^2]$ such that $\mathbb{E}_{\mathbf{x}\sim\mathcal{S}}[r_{\mathbf{x}}] = K\bar{L}_{\mathcal{S}}(\mathcal{P})$. Before that, we prove the following lemma.

\begin{lemma}\label{jensen}
Let $R$ be a random variable which takes values in $\{0,\dots,K\}$ with $\mathbb{E}[R] = \mu$. If $\delta = \mu - \lfloor\mu\rfloor$, then:
\[
\mathbb{E}[R^2]\ge (1-\delta)\lfloor\mu\rfloor^2+\delta(\lfloor\mu\rfloor+1)^2.
\]
\end{lemma}
\begin{proof}
Let $\tilde{f}$ be the piecewise-linear interpolation of $f(r)=r^2$ on $\{0,\dots,K\}$. Since $f$ is convex, so is $\tilde{f}$. Applying Jensen's inequality to $\tilde{f}$, we obtain:
\[
\mathbb{E}[R^2] = \mathbb{E}[\tilde{f}(R)] \geq \tilde{f}(\mathbb{E}[R]) = \tilde{f}(\mu) \tag{4}
\]
Since \(\mu\in[\lfloor\mu\rfloor,\lfloor\mu\rfloor+1]\) and \(\mu=(1-\delta)\lfloor\mu\rfloor+\delta(\lfloor\mu\rfloor+1),\) the linearity of \(\tilde f\) on this interval gives
\[
\tilde f(\mu)
=
(1-\delta)\lfloor\mu\rfloor^2
+
\delta(\lfloor\mu\rfloor+1)^2 \tag{5}
\]
Combining (4) and (5) completes the proof of this lemma.
\end{proof}

\noindent Now, using Lemma~\ref{jensen} with $R = r_\mathbf{x},\;\mu = K\bar{L}_{\mathcal{S}}(\mathcal{P})$ and $\delta = \mu-\lfloor\mu\rfloor$, we get:
\begin{align}
\mathbb{E}_{\mathbf{x}\sim\mathcal{S}}[r_{\mathbf{x}}^2] &\geq (1-\delta)\lfloor\mu\rfloor^2 + \delta(\lfloor\mu\rfloor+1)^2 \notag\\
&\geq \mu^2 + \delta(1-\delta)\tag{6}
\end{align}
Combining (3) and (6), we get:
\begin{align}
\mathbb{E}_{\mathbf{x}\sim\mathcal{S}}[SC_{\mathcal{P}}(\mathbf{x})]
&\geq
1-\frac{1}{\binom{K}{2}}
\left(
K^2\bar{L}_{\mathcal{S}}(\mathcal{P})
-
\mu^2 - \delta(1-\delta)
\right)\notag\\
&\geq
1-\frac{1}{\binom{K}{2}}
\left[
\mu(K-\mu)-\delta(1-\delta)
\right]\notag
\end{align}

To see that the bound is tight, construct a dataset $\mathcal{S}$ in which a fraction \(1-\delta\) of instances are misclassified by exactly \(\lfloor\mu\rfloor\) models, and a fraction \(\delta\) of instances are misclassified by exactly \(\lfloor\mu\rfloor+1\) models. Then
\(\mathbb{E}_{\mathbf{x}\sim\mathcal{S}}[r_{\mathbf{x}}] =
(1-\delta)\lfloor\mu\rfloor + \delta(\lfloor\mu\rfloor+1) = \mu, \)
so the average empirical risk of the model pool is
\(\bar L_{\mathcal{S}}(\mathcal{P})=\mu/K.\)
Moreover,
\[
\mathbb{E}_{\mathbf{x}\sim\mathcal{S}}[r_{\mathbf{x}}^2]
=
(1-\delta)\lfloor\mu\rfloor^2
+
\delta(\lfloor\mu\rfloor+1)^2
=
\mu^2+\delta(1-\delta).
\]
Substituting this in (3), we get:
\[
\mathbb{E}_{\mathbf{x}\sim\mathcal{S}}[SC_{\mathcal{P}}(\mathbf{x})]
=
1-\frac{1}{\binom{K}{2}}
\left[
\mu(K-\mu)-\delta(1-\delta)
\right]
\]
which matches the lower bound. Hence, the bound is tight.
\end{proof}

\section{Self-Consistency Bound for Subgroups }
\label{supp:self-consistency-subgroups}

As a corollary of Proposition~\ref{supp:TLB}, the lower-bound for expected self-consistency of a finite model pool $\mathcal{P}$ at a subgroup-level is simply:
\[
\mathbb{E}_{\mathbf{x}\sim G}
\left[
SC_{\mathcal{P}}(\mathbf{x})
\right]
\geq
1-\frac{1}{\binom{K}{2}}
\left[
\mu(K-\mu)-\delta(1-\delta)
\right]
\]
where $\mu = K\bar{L}_G(\mathcal{P})$ and $\delta = \mu - \lfloor\mu\rfloor$. In Figure~\ref{fig:subgroupplots}, we compare the observed self-consistency of $\mathcal{P}$ against its worst-case guarantee for all subgroups in our data. We find that observed self-consistency is consistently higher than its theoretical lower bound across all subgroups and all values of $\epsilon$. Thus, despite its tightness, the bound is quite conservative in practice.

\begin{figure*}[h]
    \centering\small

    \begin{subfigure}[b]{0.31\textwidth}
        \centering
        \includegraphics[width=\textwidth]{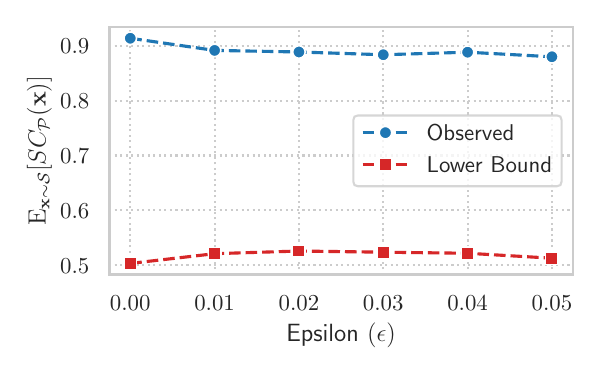}
        \caption{Male, Age$<$30, National; Size=173}
        \label{fig:1}
    \end{subfigure}
    \hfill
    \begin{subfigure}[b]{0.31\textwidth}
        \centering
        \includegraphics[width=\textwidth]{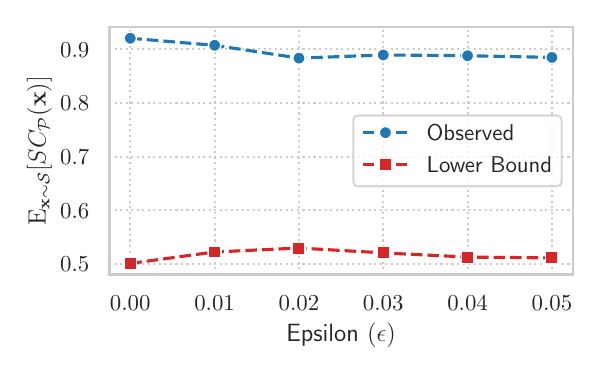}
        \caption{Male, Age$<$30, Foreigner; Size=142}
        \label{fig:2}
    \end{subfigure}
    \hfill
    \begin{subfigure}[b]{0.31\textwidth}
        \centering
        \includegraphics[width=\textwidth]{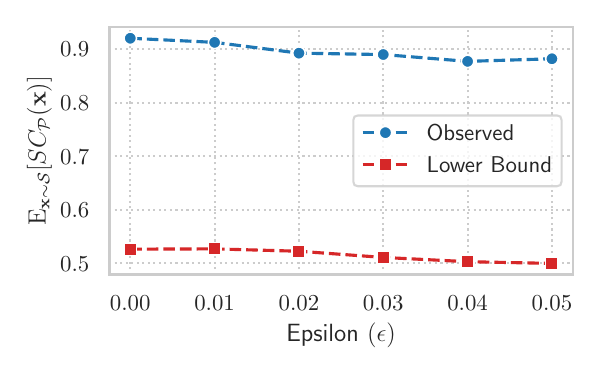}
        \caption{Male, Age$\geq$30, National; Size=982}
        \label{fig:3}
    \end{subfigure}

    \vspace{0.5em}

    \begin{subfigure}[b]{0.31\textwidth}
        \centering
        \includegraphics[width=\textwidth]{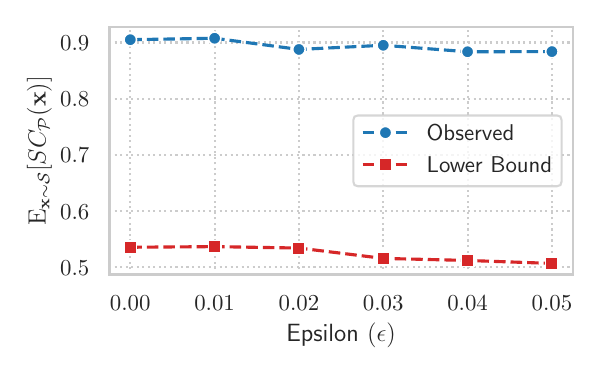}
        \caption{Male, Age$\geq$30, Foreigner; Size=501}
        \label{fig:4}
    \end{subfigure}
    \hfill
    \begin{subfigure}[b]{0.31\textwidth}
        \centering
        \includegraphics[width=\textwidth]{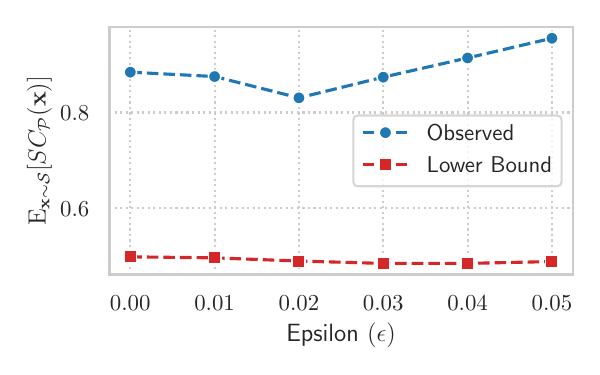}
        \caption{Female, Age$<$30, National; Size=6}
        \label{fig:5}
    \end{subfigure}
    \hfill
    \begin{subfigure}[b]{0.31\textwidth}
        \centering
        \includegraphics[width=\textwidth]{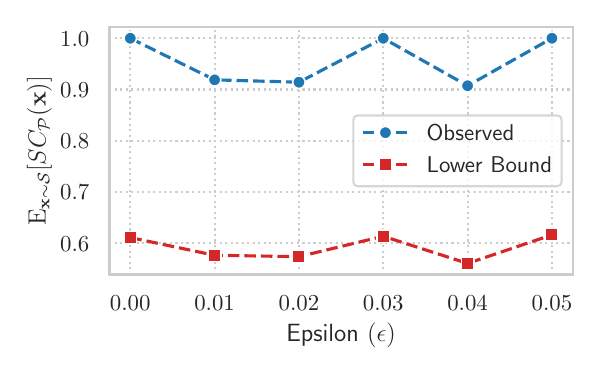}
        \caption{Female, Age$<30$, Foreigner; Size=4}
        \label{fig:6}
    \end{subfigure}

    \vspace{0.5em}

    \begin{subfigure}[b]{0.31\textwidth}
        \centering
        \includegraphics[width=\textwidth]{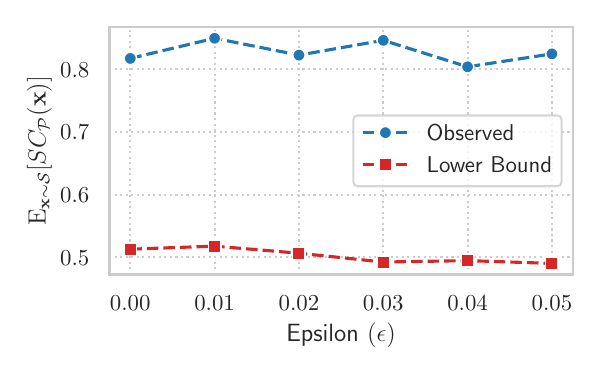}
        \caption{Female, Age$\geq$30, National; Size=61}
        \label{fig:7}
    \end{subfigure}
    \begin{subfigure}[b]{0.31\textwidth}
        \centering
        \includegraphics[width=\textwidth]{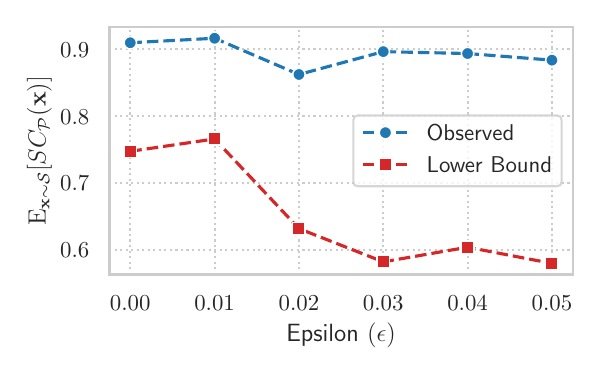}
        \caption{Female, Age$\geq$30, Foreigner; Size=12}
        \label{fig:8}
    \end{subfigure}

    \caption{Expected self-consistency of the model pool returned by \algoname{FairSLIM} on the \textbf{test set} at a subgroup-level versus $\epsilon$. While both quantities vary with $\epsilon$, a large gap between the observed values and the lower bound persists throughout. Interpretation for the smaller subgroups should be considered cautiously due to limited sample sizes.}
    \label{fig:subgroupplots}
\end{figure*}

\section{Self-Consistency and Pool Size}
\label{supp:self-consistency-poolsize}

Figures~\ref{fig:pooltrain} and~\ref{fig:pooltest} show the sensitivity of self-consistency of a model pool $\mathcal{P}$ with respect to its size. To produce these figures, we sample 10 subsets of $k$ models
uniformly at random from $\mathcal{P}$ for every combination of $k \in \{5,10,15,20,25\}$ and $\epsilon \in \{.01,.02,.03,.04,.05\}$. Then, for every sample of size $k$, we compute its average self-consistency on both the train and test sets. We observe that the overall trend that average self-consistency remains substantially higher across values of $\epsilon$ than worst case analysis suggests, is invariant to pool size for $k \leq 25$.

\begin{figure*}
    \centering
    \includegraphics[width=\linewidth]{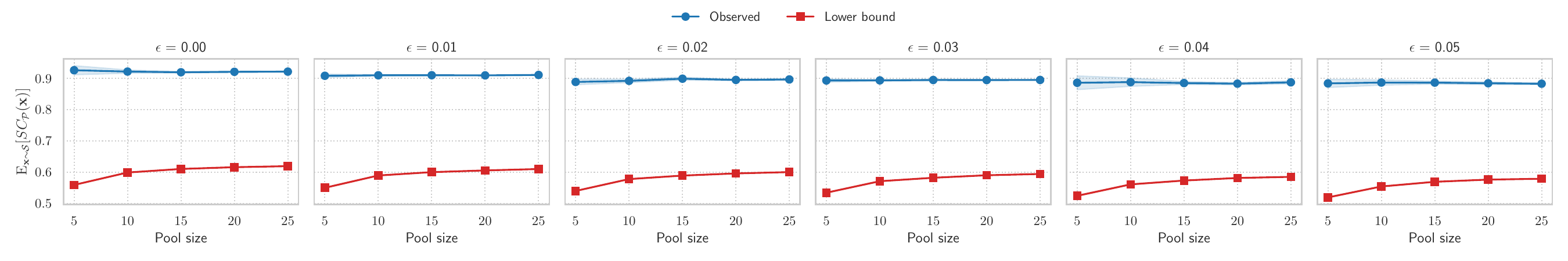}
    \caption{Expected self-consistency of the model pool returned by \algoname{FairSLIM} on the \textbf{train set} as a function of varying Pool Sizes. The large gap between the observed self-consistency and its theoretical lower-bound is persistent throughout.}
    \label{fig:pooltrain}
\end{figure*}

\begin{figure*}
    \centering
    \includegraphics[width=\linewidth]{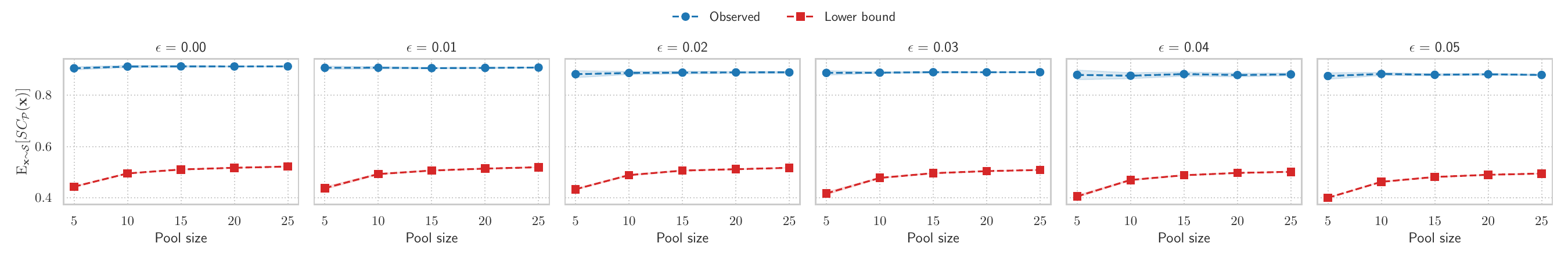}
    \caption{Expected self-consistency of the model pool returned by \algoname{FairSLIM} on the \textbf{test set} as a function of varying Pool Sizes. The large gap between the observed self-consistency and its theoretical lower-bound is persistent throughout.}
    \label{fig:pooltest}
\end{figure*}

\begin{table*}[t!]
\centering\small
\caption{Overview of the 20 static and 23 dynamic features used in \algoname{\RisCanvi} for risk assessment.}
\label{tab:riscanvifeatures}
\begin{tabular}{@{}clc@{}}
\toprule
\textbf{Feature} & \multicolumn{1}{l}{\textbf{Feature Name}}                         & \textbf{Feature Type} \\ \midrule
F1               & Violent Base Crime                                                & Static                \\
F2               & Age at the time of Base Crime                                     & Static                \\
F3               & Intoxication during the Commission of Base Crime                  & Static                \\
F4               & Number of Victims with Injuries                                   & Static                \\
F5               & Duration of the Penalty                                           & Static                \\
F6               & Uninterrupted Time in Prison                                      & Dynamic               \\
F7               & History of Violence                                               & Static                \\
F8               & Age for Start of Criminal or Violent Activity                     & Static                \\
F9               & Increase in the Frequency, Severity and/or Diversity of Crimes    & Static                \\
F10              & Conflicts with Other Inmates                                      & Dynamic               \\
F11              & Non-Compliance with Criminal Measures                             & Dynamic               \\
F12              & Disciplinary Files                                                & Static                \\
F13              & Evasions or Escapes                                               & Static                \\
F14              & Regression of Degree                                              & Static                \\
F15              & Break of Permissions                                              & Dynamic               \\
F16              & Childhood Maladjustment                                           & Static                \\
F17              & Distance between Residence and Penitentiary                       & Dynamic               \\
F18              & Level of Education                                                & Dynamic               \\
F19              & Problems in Employment                                            & Static                \\
F20              & Lack of Financial Resources                                       & Dynamic               \\
F21              & Absence of Viable Future Plans                                    & Dynamic               \\
F22              & Criminal History in the Family of Origin                          & Static                \\
F23              & Problematic Socialization in the Family of Origin                 & Static                \\
F24              & Lack of Family and Social Support                                 & Dynamic               \\
F25              & Criminal Friendships                                              & Dynamic               \\
F26              & Belongs to Social Risk Groups                                     & Dynamic               \\
F27              & Prominent Criminal Role                                           & Dynamic               \\
F28              & Victim of Gender Violence (only applicable to women)              & Dynamic               \\
F29              & Current Family Burdens                                            & Dynamic               \\
F30              & Drug Abuse or Dependence                                          & Dynamic               \\
F31              & Alcohol Abuse or Dependence                                       & Dynamic               \\
F32              & Severe Mental Disorder                                            & Static                \\
F33              & Promiscuous Sexual Behavior or Paraphilia                         & Static                \\
F34              & Limited/No Response to Psychological and/or Psychiatric Treatment & Dynamic               \\
F35              & Personality Disorder related to Anger, Impulsivity or Violence.   & Static                \\
F36              & Poor Coping with Stress                                           & Dynamic               \\
F37              & Attempts or Behaviors of Self-Harm                                & Static                \\
F38              & Pro-Criminal Attitudes or Anti-Social Values                      & Dynamic               \\
F39              & Low Mental Capacity and Intelligence                              & Static                \\
F40              & Recklessness                                                      & Dynamic               \\
F41              & Impulsivity and Emotional Instability                             & Dynamic               \\
F42              & Hostility                                                         & Dynamic               \\
F43              & Irresponsibility                                                  & Dynamic               \\ \bottomrule
\end{tabular}
\end{table*}

\end{document}